\documentclass[11pt]{article}

\usepackage{graphicx}

\usepackage{amsthm,amsfonts}
\usepackage{amssymb,mathrsfs}
\usepackage[cmex10]{amsmath}

\usepackage{multirow}
\usepackage{tikz}
\usepackage{bm}
\usepackage{soul}

\usepackage{algorithm}
\usepackage{algpseudocode}
\algnewcommand\algorithmicinput{\textbf{Input:}}
\algnewcommand\INPUT{\item[\algorithmicinput]}
\algnewcommand\algorithmicoutput{\textbf{Output:}}
\algnewcommand\OUTPUT{\item[\algorithmicoutput]}

\usepackage{caption}
\usepackage{subcaption}

\makeatletter
\newcommand{\specialcell}[1]{\ifmeasuring@#1\else\omit$\displaystyle#1$\ignorespaces\fi}
\makeatother

\def\transpose{{\hbox{\tiny\it T}}}

\renewcommand{\Re}{{\mathbb R}}

\newcommand{\argmin}{\mathop{\rm arg\, min}}

\def\be{\begin{eqnarray}}
\def\ee{\end{eqnarray}}
\def\ben{\begin{eqnarray*}}
\def\een{\end{eqnarray*}}

\def\ddxp{{\mathchoice{\FRAC{1}{d^+}{dx}}%
{\FRAC{1}{d^+}{dx}}%
{\FRAC{3}{d^+}{dx}}%
{\FRAC{3}{d^+}{dx}}}}

\graphicspath{{PDF/}}
\def\Ebox#1#2{%
\begin{center}
\includegraphics[width= #1\hsize]{#2} 
\end{center}}

\def\sq{$\Box$}

\def\qed{\ifmmode\Box\else{\unskip\nobreak\hfil
\penalty50\hskip1em\null\nobreak\hfil\sq
\parfillskip=0pt\finalhyphendemerits=0\endgraf}\fi\par\medbreak}

\newcommand{\bsq}{\rule{1.25ex}{1.25ex}}

\def\bqed{\ifmmode\bsq\else{\unskip\nobreak\hfil
\penalty50\hskip1em\null\nobreak\hfil\bsq
\parfillskip=0pt\finalhyphendemerits=0\endgraf}\fi\medskip}

\newsavebox{\junk}
\savebox{\junk}[1.6mm]{\hbox{$|\!|\!|$}}

\newcommand{\field}[1]{\mathbb{#1}}

\def\ZZ{\field{Z}}

\def\elig{\zeta}
\def\bfelig{\bfmath{\zeta}}

\newcommand{\one}{\hbox{\rm\large\textbf{1}}}

\def\bfmN{{\mbox{\protect\boldmath$N$}}} 

\def\bfmU{{\mbox{\protect\boldmath$U$}}}

\def\bftheta{\mbox{\boldmath$\theta$}}

\def\bfvarphi{\mbox{\boldmath$\varphi$}}

\def\til={{\widetilde =}}

\def\clA{{\cal A}}

\def\clE{{\cal E}}

\def\clN{{\cal N}}

\def\half{{\mathchoice{\textstyle \frac{1}{2}}%
{\frac{1}{2}}%
{\hbox{\tiny $\frac{1}{2}$}}%
{\hbox{\tiny $\frac{1}{2}$}} }}

\def\eqdef{\mathbin{:=}}

\def\Prob{{\sf P}}

\def\Expect{{\sf E}}

\def\epsy{\varepsilon}

\def\varble{\,\cdot\,}

\newtheorem{theorem}{Theorem}[section]

\newtheorem{proposition}[theorem]{Proposition}
\newtheorem{lemma}[theorem]{Lemma}

\def\Lemma#1{Lemma~\ref{#1}}
\def\Prop#1{Prop.~\ref{#1}}

\def\Section#1{Section~\ref{#1}}

\def\Figure#1{Figure~\ref{#1}} 

\newcommand{\oo}{\overline}

\def\barh{{\oo {h}}}

\def\Lvx{L_\infty^v}

\def\Lvx{L_\infty^V}

\def\Lvx{L_\infty^{v}}

\def\tilM{\widetilde{M}}

\def\bfmX{{\mbox{\protect\boldmath$X$}}}
\def\bftheta{{\mbox{\protect\boldmath$\theta$}}}

\def\FRAC#1#2#3{\genfrac{}{}{}{#1}{#2}{#3}}

\def\half{{\mathchoice{\FRAC{1}{1}{2}}%
{\FRAC{1}{1}{2}}%
{\FRAC{3}{1}{2}}%
{\FRAC{3}{1}{2}}}}

\def\tilpsi{\widetilde \psi}

\def\grad{\nabla}
\def\gradpsi{\nabla \psi}

\def\tilgrad{\nabla^\Sens\!}

\def\bfmath#1{{\mathchoice{\mbox{\boldmath$#1$}}%
{\mbox{\boldmath$#1$}}%
{\mbox{\boldmath$\scriptstyle#1$}}%
{\mbox{\boldmath$\scriptscriptstyle#1$}}}}

\def\bfmN{{\mbox{\protect\boldmath$N$}}}

\def\transpose{{\hbox{\tiny\it T}}}

\newcounter{rmnum}
\newenvironment{romannum}{\begin{list}{{\upshape (\roman{rmnum})}}{\usecounter{rmnum}
\setlength{\leftmargin}{12pt}
\setlength{\rightmargin}{10pt}
\setlength{\itemindent}{-1pt}
}}{\end{list}}

\newcounter{anum}

\newenvironment{Anum}[1]{\smallbreak
\noindent
{\bf{Assumption~{#1}}:}
\begin{list}{{\upshape \textbf{#1.\arabic{anum}}:}}{\usecounter{anum}
\setlength{\leftmargin}{12pt}
\setlength{\rightmargin}{12pt}
\setlength{\itemindent}{22pt}
}}{\end{list}}

\newlength{\noteWidth}
\long\def\notes#1{\ifinner
{\tiny #1}
\else
\marginpar{\parbox[t]{\noteWidth}{\raggedright\tiny #1}}
\fi}

\def\DeltaA{\Delta}  

\def\fee{\text{f}}

\def\dffA{{\cal A}}
\def\Sens{{\cal S}}

\def\tilM{\widetilde{M}}
\def\tilb{\widetilde{b}}

\newcommand{\Chi}{{%
\mathchoice{\mathord{\raisebox{1.5pt}{\scalebox{1.25}{$\chi$}}}}%
{\mathord{\raisebox{1.5pt}{\scalebox{1.25}{$\chi$}}}}%
{\mathord{\raisebox{1pt}{\scalebox{.75}{$\chi$}}}}%
{\mathord{\raisebox{.8pt}{\scalebox{.6}{$\chi$}}}}%
}}

\DeclareBoldMathCommand\bfSens{{\cal S}}

\begin{document}

\title{Differential Temporal Difference Learning 
	\\
}


\author{Adithya. M. Devraj
\thanks{Department of Electrical and Computer 
Engineering,
University of Florida, Gainesville, USA.
Email: {\tt adithyamdevraj@ufl.edu}}
\and
Ioannis Kontoyiannis
\thanks{
Department of Engineering,
University of Cambridge,
Trumpington Street, Cambridge CB2 1PZ, UK.
Email: {\tt i.kontoyiannis@end.cam.ac.uk}.          
}
\and
Sean. P. Meyn
\thanks{Department of Electrical and Computer 
Engineering,
University of Florida, Gainesville, USA.
Email: {\tt meyn@ece.ufl.edu}.
}}

\maketitle
\thispagestyle{empty}

\begin{abstract}

Value functions derived from Markov decision processes arise as a
central component of algorithms as well as performance metrics in many
statistics and engineering applications of machine learning.
Computation of the solution to the associated Bellman equations is
challenging in most practical cases of interest. A popular class of
approximation techniques, known as Temporal Difference (TD) learning
algorithms, are an important sub-class of general reinforcement
learning methods. The algorithms introduced in this work are intended
to resolve two well-known issues with TD-learning algorithms: Their
slow convergence due to very high variance, and the fact that, for the
problem of computing the relative value function, consistent algorithms
exist only in special cases. First we show that the gradients of these
value functions admit a representation that lends itself to algorithm
design. Based on this result, a new class of \textit{differential TD-learning}
algorithms is introduced. For Markovian models on Euclidean space with
smooth dynamics, the algorithms are shown to be consistent under
general conditions. Numerical results show dramatic variance reduction
in comparison to standard methods.

\medskip


\end{abstract}

\newpage

%

%
%
%
%

\section{Introduction}
\label{s:Intro}

A central task in the application of many machine learning methods
and control techniques is the (exact or approximate) computation 
of value functions arising from Markov decision processes. 
The class of {\em Temporal Difference} 
(TD) {\em learning} algorithms considered in this work 
is an important sub-class 
of the general family of {\em reinforcement learning} methods that performs this task. Our main 
contributions here are the introduction of a related family of 
TD-learning algorithms that enjoy better convergence properties 
than existing methods, and the rigorous theoretical analysis
of these algorithms.

The value functions considered in this work are based 
on a discrete-time Markov chain  
$\bfmX = \{X(t): t=0,1,2,\ldots\}$ taking values in $\Re^\ell$, 
and on an associated {\em cost function} $c:\Re^\ell \to \Re$.  
Our central modelling assumption throughout is that $\bfmX$
evolves according to the nonlinear state space model,
\begin{equation}
X(t+1) = a(X(t),N(t+1)), \quad t \geq 0,
\label{e:SP_disc}
\end{equation}
where $\bfmN \!=\! \{N(t)\!:\! t \! = \!0,1,2,\ldots\}$ is an $m$-dimensional 
disturbance sequence of independent and identically distributed
(i.i.d.) random variables, 
and  $a: \Re^{\ell + m} \to \Re^\ell$ is a continuous mapping.    
Under these assumptions, for all $t \geq 0$, $X(t+1)$ is a continuous function of 
the initial condition $X(0)=x$;  this observation is our starting 
point for the construction of effective algorithms for value function 
approximation.

We begin with some familiar background.

\subsection{Value functions}

Given a discount factor $\beta \in (0,1)$, the 
{\em discounted-cost value function} is defined as
\begin{equation}
h_\beta(x) \eqdef  \sum_{t=0}^{\infty} 
\beta^t \Expect [c(X(t)) \mid X(0) = x]\,, \qquad x \in \Re^\ell.
\label{e:DCOE_disc}
\end{equation}
It is known that $h_\beta$ solves the {\em Bellman equation} \cite{bertsi96a,CTCN}:
\begin{equation}
c(x) + \beta \Expect[h_\beta(X(t  +  1)) | X(t)  =  x ] - h_\beta (x)   = 0\,.
\label{e:DPEquDisc}
\end{equation}
The {\em average cost} is defined as the ergodic limit,
\begin{equation}
\eta \eqdef \lim_{n\rightarrow\infty} 
\frac{1}{n} \sum_{t=0}^{n-1} \Expect [c(X(t)) \mid X(0) = x]\, ,
\label{e:barc}
\end{equation}
where the limit exists and is independent of $x$ under the conditions 
discussed in \Section{s:Rep}. The following \emph{relative value function} 
is central to analysis of average cost control problems:
\begin{equation}
h(x) \eqdef  \sum_{t=0}^{\infty} \Expect [c(X(t))-\eta \mid X(0) = x]\,, 
\qquad x \in \Re^\ell.
\label{e:fish-sum}
\end{equation}
Provided the sum \eqref{e:fish-sum} exists for each $x$,  the   relative value function solves the  \emph{Poisson equation}  \cite{tsiroy99a,CTCN}: 
\begin{equation} 
\Expect[h(X(t+1)) \mid X(t) = x ]  - h(x) = - [c(x)-\eta]\, .
\label{e:ACOE_disc2}
\end{equation} 
These equations and their solutions are of interest in learning theory, 
control engineering, and many other fields, including:

\textit{Optimal control and Markov decision processes:}   
Policy iteration and actor-critic algorithms are designed to approximate an  
optimal policy using two-step procedures: First, given a policy, the 
associated value function is computed (or approximated), and then the policy is updated based 
on this value function \cite{bershr96a, kontsi03a}. These approaches can 
be used for both discounted- and average-cost optimal control problems.

\textit{Algorithm design for variance  reduction:}  
Under general conditions, the asymptotic variance (i.e., the variance 
appearing in the central limit theorem for the averages
in \eqref{e:barc}) is naturally expressed in terms of the relative value 
function $h$ \cite{asmgly07,MT}.  The method of control 
variates is intended to reduce the asymptotic variance of various 
Monte Carlo methods; a version of this technique involves 
the construction of 
an approximation to $h$
\cite{HenThesis,henmeytad03a,kyrkonmey08,delkon12,brodurmeymourad19}.
 
\textit{Nonlinear filtering:}  A recent approach to approximate nonlinear 
filtering requires the \emph{gradient} of the solution to Poisson's equation to obtain the 
``innovation gain" \cite{yanmehmey13,laumehmeyrag15}.  Approximations of 
the solution can lead to 
efficient implementations of this method \cite{tagmeh16a,raddevmey16,radmey18a}.

\subsection{TD-learning and value function approximation}

In most cases of practical interest, closed-form expressions for the value 
functions $h_\beta$ and $h$ in \eqref{e:DCOE_disc} or \eqref{e:ACOE_disc2}
cannot be derived. One approach to obtaining approximations is 
the   {\em Temporal Difference} (TD)
{\em learning} algorithm \cite{sutbar98,bertsi96a}.

In the case of the discounted-cost value function,
the goal of TD-learning is to approximate~$h_\beta$ 
as a member of a parametrized family of 
functions $\{h_\beta^\theta: \theta\in\Re^d\}$.   
Throughout, we restrict attention
to linear parametrizations of the form,
\begin{equation}
h_\beta^\theta = \sum_{j=1}^d \theta_j   \psi_j,
\label{e:hLinearPar}
\end{equation}
where we write
$\theta=(\theta_1,\theta_2,\ldots,\theta_d)^\transpose$, 
$\psi=(\psi_1,\psi_2,\ldots,\psi_d)^\transpose$, and we assume
that the given collection of
`basis' functions $\psi\colon\Re^\ell\to\Re^d$ is continuously differentiable.

For any function $f: \Re^\ell \to \Re$, and a given probability measure $\mu:\Re^\ell \to [0, 1]$, we define the $\mu$-norm:
\[
\begin{aligned}
\| f \|_{\mu} & \eqdef \Big( \Expect[ f^2(X) ] \Big)^{\half} \,, \quad X \sim \mu
\\
& = \Big( \displaystyle \int  f^2 (x) \mu(dx) \Big)^{\half}
\end{aligned}
\]
In one variant of the TD technique 
(the LSTD(1) algorithm, described in Section~\ref{s:dtdlambda}),
the optimal parameter vector~$\theta^*$ is 
chosen as 
the solution to a minimum-norm problem,
\begin{equation}
\begin{aligned}
\theta^* &= \argmin_\theta \| h_\beta^\theta - h_\beta\|^2_\pi 
\end{aligned}
\label{e:TDgoal}
\end{equation}
where the expectation is with respect to $X\sim \pi$,
and  $\pi$ denotes the steady-state distribution of the Markov chain $\bfmX$;
more details are provided in Sections~\ref{s:setup}
and~\ref{s:dtdlambda}.

A TD-learning algorithm is said to be \emph{consistent}, if the parameter estimates obtained using the algorithm converge to $\theta^*$.

Theory for TD-learning in the discounted-cost setting is largely complete, 
in the sense that criteria for convergence are well-understood, and the 
asymptotic variance of the algorithm is computable based on standard theory 
from stochastic approximation \cite{bor08a,devmey17a,devmey17b}
; see \cite{chedevbusmey19f,devmey20a,devbusmey20} for the relationship between asymptotic variance and convergence rate of TD-learning algorithms.
Theory and algorithms for the average-cost setting involving the
relative value function $h$ is more fragmented.  The 
optimal parameter $\theta^*$ in the analog of \eqref{e:TDgoal} 
with $h_\beta$ replaced by the relative value function $h$
can be computed using TD-learning techniques only for Markovian 
models that regenerate:  
there exists a state $x^\bullet \in \Re^\ell$ that is visited infinitely often 
(see the regenerative TD($1$) algorithm in \cite[pg. 1012]{kontsi00}, and also \cite{CTCN, huachemehmeysur11}).
 
Regeneration is often not a restrictive assumption.   However,  
the asymptotic variance of these algorithms grows with the variance of 
inter-regeneration  times (time between consecutive visits to the regenerative state $x^\bullet$).
The variance can be massive even in simple examples such as the M/M/1 queue;
see the final chapter of \cite{CTCN}. High variance is 
also predominantly observed in the discounted-cost case 
when the discounting factor is close to $1$;
see the relevant remarks in Section~\ref{s:summary}.

The \textit{differential} TD-learning 
algorithms developed in 
this paper are designed in part to resolve these issues.  The main idea is to estimate 
the  \emph{gradient} of the value function.  Under the conditions imposed, 
the asymptotic variance of the resulting algorithms remains uniformly bounded  
over $0 < \beta < 1$. And the same techniques can be applied to obtain 
finite-variance algorithms for approximating the 
relative value function $h$
for models without regeneration.   

We note that
the needs of the analysis of the algorithms 
presented here have, in part, motivated the development of rich new
convergence
theory for general classes of Markov processes
\cite{devkonmey17a}.
Indeed, the results in Sections~\ref{s:Rep}
and~\ref{s:TD} draw heavily on the 
convergence results established
in \cite{devkonmey17a}.

\subsection{Differential TD-learning}

Consider the discounted-cost setting; suppose that the value function
$h_\beta$ and 
all its potential approximations $\{h_\beta^\theta: \theta\in\Re^d\}$ 
are continuously differentiable
as functions of the state $x$,
i.e., $h_{\beta},\, h_{\beta}^{\theta} \, 
\in C^1$, for each $\theta \in \Re^d$.
In terms of the
linear parametrization~\eqref{e:hLinearPar},
we obtain approximations of the form:
\begin{equation}
\grad h_\beta^\theta = \sum_{j=1}^d \theta_j \grad \psi_j\,,
\label{e:gradhtheta}
\end{equation} 
where the gradient is with respect to the state $x$.

The {\em differential} LSTD-learning algorithm 
introduced in Section~\ref{s:TD}
is designed to compute 
the solution to  
\begin{equation}
\theta^* =   \argmin_\theta \Expect[ \| \grad h_\beta^\theta(X) 
- \grad h_\beta(X)\|_2^2 ]
\, ,\quad X\sim \pi\,,
\label{e:gradTD}
\end{equation}
where $\|\cdot\|_2$ is the usual Euclidean norm, and once again, $\pi$ denotes the steady-state distribution of the Markov chain $\bfmX$.

The value function approximation
$h_\beta^{\theta^*} $  is obtained via the addition of a constant:
\begin{equation}
h_\beta^{\theta^*} = \sum_{j=1}^d \theta_j ^*  \psi_j + \kappa (\theta^*).
\label{e:gradTDaddConstant}
\end{equation}
The mean-square optimal choice is obtained on requiring
\begin{equation}
\Expect[  h_\beta^{\theta^*}(X) -   h_\beta(X) ] =0
\, ,\quad X\sim \pi\,.
\label{e:LSTD_kappa_condition}
\end{equation}
{See the discussion that follows Algorithm~\ref{dLSTD} 
for details.}

A similar program can be carried out for 
the relative value function $h$, which,
viewed as a solution to Poisson's 
equation~(\ref{e:ACOE_disc2}),
is unique only up to an additive constant.
Therefore,
we can set $\kappa (\theta^*) =0$ in the average-cost setting.

\subsection{Summary of contributions}
\label{s:summary}
The main contributions of this work are:

$(i)$~The introduction of the new
{\em differential} Least 
Squares TD-learning ($\nabla$-LSTD, or `grad-LSTD') algorithm,
which is applicable
in both the discounted- and average-cost settings.

$(ii)$~The development of appropriate conditions under which
we can show that,
for linear parametrizations, 
$\nabla$-LSTD converges and
solves the quadratic 
program \eqref{e:gradTD}. 	

$(iii)$~The introduction of the family of 
$\nabla$-LSTD($\lambda$)-learning algorithms.
It is shown that $\nabla$-LSTD($1$) also solves 
the quadratic program \eqref{e:gradTD}.

 $(iv)$ These new algorithms are applicable for models 
that do not have regeneration.   
Their asymptotic variance is uniformly bounded over all $0<\beta<1$,
under general conditions.  

Perhaps the most important limitation of the $\nabla$-LSTD algorithms is the requirement of partial knowledge of the Markov chain 
transition dynamics.  \Section{s:conclusions} contains discussion on how to address this challenge.   Fortunately, in many applications, very little knowledge is required  
(such as in the queuing example discussed in \Section{s:dyn_ss}).

Finally, a few more remarks about the error rates of these algorithms
are in order.
From the definition of the value function \eqref{e:DCOE_disc},  it 
can be expected that  $ h_\beta(x)  \to \infty$   as $\beta \to 1$  for each $x\in\Re^\ell$.    
This is why approximation methods in reinforcement learning typically 
take for granted that error will grow at this rate.   Moreover, it is 
observed that variance in reinforcement learning can grow dramatically 
with the discount factor.  In particular,  it is shown in   
\cite{devmey17a,devmey20a} that asymptotic variance in the standard Q-learning algorithm 
of Watkins is \textit{infinite} when the
discount factor satisfies $\beta>1/2$.  

The family of TD($\lambda$) algorithms was introduced in \cite{sutbar98}
to reduce 
the variance of earlier methods, but 
it brings its own potential challenges.
Consider \cite[Theorem 1]{tsiroy97a},  which compares the estimate 
$ h^{\theta^\lambda}_ \beta $ obtained using TD($\lambda$),  
with the $L_2$-optimal approximation 
$ h^{\theta^*}_ \beta $ obtained using TD($1$):
\begin{equation}
\| h^{\theta^\lambda}_ \beta  -  h_\beta \|_{\pi} 
	\le \frac{1-\lambda \beta}{1-\beta}  \| h^{\theta^*}_ \beta  
	-  h_\beta \|_{\pi}.
\label{e:tsiroy97a_Theorem1}
\end{equation}
This bound suggests that the bias can grow 
as $(1-\beta)^{-1}$. 
 
The difficulties are more acute when we come 
to the average-cost problem. Consider the 
minimum-norm problem~(\ref{e:TDgoal})
with the relative 
value function $h$ in place of $h_\beta$:
\begin{equation} 
\theta^* = \argmin_\theta \| h^\theta - h\|^2_\pi .
\label{e:TDgoalAC}
\end{equation}
Here, for the TD($\lambda$) algorithm with $\lambda<1$,
Theorem~3 of \cite{tsiroy99a} implies a bound in 
terms of the ``convergence rate" $\rho$ for the Markov chain,
 \begin{equation}
\| h^{\theta^\lambda}_ \beta  -  h \|_{\pi} 
	\le  c(\lambda,\rho)  \| h^{\theta^*}  -  h \|_{\pi},
\label{e:tsivan99a_Theorem3}
\end{equation}
in which $ c(\lambda,\rho)>1$ and
$c(\lambda,\rho)\to 1$ as $\lambda\to 1$.   

Convergence of TD($1$) does not hold for the average cost problem.   Specialized algorithms making use of regeneration were introduced in 
\cite[pg. 1012]{kontsi00}, \cite{CTCN, huachemehmeysur11}.

For any differentiable function function $f: \Re^\ell \to \Re$, its gradient is denoted
\begin{equation}
\nabla f  \eqdef \Big [\frac{\partial }{\partial x_1} f \,, \ldots \,, \frac{\partial }{\partial x_\ell} f \Big ]^\transpose
\end{equation}
Under the assumptions imposed in this paper, 
we show that the gradients of the value 
functions are well behaved: $\{\nabla h_\beta : 0 < \beta \leq 1\}$ is a bounded 
collection of functions,  and
$ \nabla h_\beta  \to \nabla h$ uniformly on compact sets.     
As a consequence, both the bias and variance of
the new $\nabla$-LSTD($\lambda$) algorithms 
are bounded over all $0<\beta \leq 1$.

The remainder of the paper is organized as follows: Basic definitions and 
value function representations are presented in \Section{s:Rep}.  
The $\nabla$-LSTD-learning algorithm is introduced in  \Section{s:TD},
 and 
the family of $\nabla$-LSTD($\lambda$) algorithms are introduced 
in \Section{s:dtdlambda}.  
Results from numerical experiments are shown in \Section{s:sim}, 
and conclusions are contained in \Section{s:conclusions}.

\section{Representations and Approximations}
\label{s:Rep}

We begin with modelling assumptions on the Markov process
$\bfmX$, and representations for the value 
functions $h_\beta,h$ and their gradients.    

\subsection{Markovian model and value function gradients}
\label{s:setup}

The evolution equation  \eqref{e:SP_disc} defines a Markov chain  $\bfmX$ 
with transition semigroup $\{P^t\}$, where $P^t(x,A)$ is defined,
for all $t\ge 0$, any state $x\in\Re^\ell$, and 
every measurable $A\subset\Re^\ell$, via,
\[
P^t(x,A):=\Prob_x\{X(t)\in A\}:=\Pr\{X(t)\in A\,|\,X(0)=x\}.
\]
For $t=1$ we write $P=P^1$,  so that:
\[
P(x,A) = \Pr\{ a(x,N(1)) \in A\}\,,
\]
where we recall that $a: \Re^{\ell + m} \to \Re^\ell$ defines the dynamics of the Markov chain in \eqref{e:SP_disc}.

The first set of assumptions ensures that the value functions 
$h_\beta$ and $h$ are 
well-defined.  Fix a continuous function  $v\colon\Re^\ell\to [1,\infty)$ 
that serves as a weighting function. For any measurable function 
$f\colon\Re^\ell\to\Re$, the $v$-norm is defined as follows:  
\[
\|f\|_v \eqdef \sup_x  \frac{|f(x)|}{v(x)} \,,
\]
and the associated Banach space is denoted  
$$
\Lvx \eqdef \{ f: \Re^\ell \to \Re : \| f  \|_v < \infty \}.
$$ 
Also, for any measurable function $f$ and measure
$\mu$, we write $\mu(f)$ for the integral, $\mu(f):=\int fd\mu.$


\begin{Anum}{A1}
\item[]
\noindent
The Markov chain $\bfmX$ is \textit{$v$-uniformly ergodic}: 
It has a unique 
invariant probability measure $\pi$, and there 
exists a continuous function $v:\Re^\ell\to\Re$
and constants $b_0 < \infty$  
and $0<\rho_0 < 1$,  such that,
for each function $f\in\Lvx$,
$$\Big| \Expect \big [f(X(t)) \mid X(0) = x \big ] 
- \pi(f) \Big| \leq b_0 \rho_0^t  \|f\|_v v(x),$$
for all $x\in\Re^\ell,\,t\ge 0$
\end{Anum}

\noindent
Assumption~A1 is not a strong assumption:  it  is essentially equivalent to geometric ergodicity in the usual sense  (ignoring a set of measure zero): combine Theorems~15.0.2 and 16.0.1 of \cite{MT}.  There is also an exact equivalence between Assumption~A1  and the existence of a ``Lyapunov function''  \cite[Theorem~16.0.1]{MT}.  See \cite{MT} and   \Section{s:sim} of this paper for examples where the assumption holds.

The following consequences are immediate \cite{MT,CTCN}: 

\begin{proposition}
\label{t:v-uni-value}
Under assumption~A1,  for any cost function $c$ such that $\|c\|_v<\infty$,
the limit  $\eta$  in \eqref{e:barc} exists with 
$\eta := \pi(c) < \infty$, 
and is independent of the initial state $x$. 
The value functions $h_\beta$ and $h$ exist as expressed
in~\eqref{e:DCOE_disc} and~\eqref{e:fish-sum},
and they satisfy equations~\eqref{e:DPEquDisc}
and~\eqref{e:ACOE_disc2},
respectively.

Moreover, there exists 
a constant $b_c<\infty$ such that the following bounds hold:
\[
\begin{aligned}
|h(x)|& \leq b_c v(x)
\\
|h_\beta(x)| & \leq b_c \big( v(x) + (1-\beta)^{-1} \big)
\\
|h_\beta(x) - h_\beta(y) | & \leq b_c \big( v(x) +  v(y)\big)
\,, \qquad\qquad  x,y\in\Re^\ell.
\end{aligned}
\] 
\end{proposition}


The following operator-theoretic notation will simplify exposition throughout.   
For any measurable function $f\colon\Re^\ell\to\Re$, the new function 
$P^t f: \Re^\ell \to \Re$ is defined as the conditional expectation:
\[
P^t f\, (x) = 
\Expect_x [f(X(t))  ]  \eqdef
\Expect [f(X(t)) \mid X(0) = x].
\]
For any $\beta\in(0,1)$, the  \textit{resolvent kernel} $R_\beta$ is   
the ``$z$-transform'' of the semigroup $\{P^t\}$,
\begin{equation}
R_\beta \eqdef \sum_{t = 0}^{\infty} \beta^t P^t.
\label{e:resolvent_disc}
\end{equation} 
Under the assumptions of \Prop{t:v-uni-value}, 
the discounted-cost value
function $h_\beta$ admits the representation,
\begin{equation}
\begin{aligned}
h_\beta = R_\beta c\, ,
\label{e:ResFormulah}
\end{aligned}
\end{equation} 
and similarly, for the relative value function $h$ we have,
\begin{equation}
h = R [c-\eta],
\label{e:ResFormulahAC}
\end{equation} 
where we write $R \equiv R_\beta$ when $\beta = 1$ 
\cite{MT,konmey03a,CTCN}.

The representations \eqref{e:ResFormulah} and \eqref{e:ResFormulahAC} are 
valuable in deriving the LSTD-learning algorithms \cite{bertsi96a,CTCN,sut88}.
Analogous representations 
for the gradients   are obtained in this paper: 
\[
\grad h_\beta =  \grad [ R_\beta c], \qquad \grad h =  \grad [ R c]\,.
\]

\subsection{Representation for the gradient of a value function} 
\label{s:rep}

In this section we describe the construction 
of operators $\Omega$ and $\Omega_\beta$, which satisfy the following:
\begin{equation}
\grad h_\beta =
\grad [ R_\beta c] =
\Omega_\beta \grad c\,,
\quad
\grad h = 
\grad [ R c] =
\Omega \grad c.
\label{e:OmDream}
\end{equation}
A more detailed account is given in \Section{s:derLSTD},
and a complete exposition of the underlying theory
together with the formal
justification of the existence and the relevant
properties of $\Omega$ and $\Omega_\beta$ can 
be found in \cite{devkonmey17a}. 

For the sake of simplicity, here we restrict 
our discussion to $h_\beta$ and its gradient. 
But it is not hard to see that the construction below 
easily generalizes to $\beta=1$; again, see
\Section{s:derLSTD} and \cite{devkonmey17a}
for the relevant details.

We require the following further assumptions:


\begin{Anum}{A2}
\item 
\noindent
The   process $\bfmN$ is independent
of $X(0)$.
\label{a:N}
\item
\noindent
The function $a$ is continuously differentiable in its first variable, 
with
\[
\sup_{x,n} \| \grad_x a (x,n)\| <\infty,
\]
where $\| \varble \|$ is any matrix norm, 
and the $\ell\times \ell$ matrix  $\grad_x a$
is defined as:
\[
[\grad a_x (x,n)]_{i,j} \eqdef \frac{\partial }{\partial x_i} (a (x,n))_j , 
\quad 1\leq i,j \leq \ell.
\]
 \label{a:nabla-a}
\end{Anum}
The first assumption, A2.\ref{a:N}, is critical so that the initial 
state $X(0)=x$ can be regarded as a   variable, with $X(t)$ a continuous function of $x$.   
This together with A2.\ref{a:nabla-a} allows us to define 
the \textit{sensitivity process} 
$\{\Sens(t)\}$, where, for each $t\geq 0$:
\begin{equation}
\Sens_{i,j} (t) \eqdef \frac{\partial X_i(t)}{\partial X_j(0)}, 
\quad 1\leq i,j \leq \ell.
\label{e:SensDef}
\end{equation}
The evolution equations \eqref{e:SP_disc} imply that the sensitivity process evolves as a random linear system,
\begin{equation}
\begin{aligned}
\Sens(t+1) &= \dffA(t+1) \Sens(t),\quad 
t\geq 0, 
\end{aligned}
\label{e:SP_sens_disc}
\end{equation}
with initial condition $\Sens(0)=I$,
where the $\ell\times \ell$ matrix
$\dffA(t)$ is defined as in
assumption A2.\ref{a:nabla-a}, by
\begin{equation}
\dffA^\transpose(t) \eqdef \grad_x a\, (X(t-1),N(t))
\label{e:dffA}
\end{equation}

For any $C^1$ function $f: \Re^\ell \to \Re$,  denote
\begin{equation}
\tilgrad f (X(t)) \eqdef \Sens^\transpose(t) \grad f(X(t))\,.
\label{e:tilgrad}
\end{equation} 
It follows from the chain rule that this coincides with the gradient of $f(X(t))$ with respect to the initial condition $x$:
\begin{equation}
\big [ \tilgrad f (X(t)) \big ]_i    = \frac{\partial  f(X(t))}{\partial{X_i(0)}},
\quad 1\leq i\leq \ell.
\label{e:nablaS_def}
\end{equation}
Equation \eqref{e:tilgrad} motivates the introduction of 
a semigroup $\{Q^t : t\ge 0\}$ of operators, 
whose domain includes functions $g:\Re^\ell\to\Re^\ell$
of the form $g = [g_{1}, \ldots, g_{\ell}]^{\transpose}$, with    
$g_i\in\Lvx$ for each $i$.  For $t=0$, $Q^0$ is the identity operator, 
and for $t\ge 1$, 
\begin{equation} 
Q^t g(x) \eqdef \Expect_x \bigl[ \Sens^\transpose(t) g(X(t))  \bigr].
\label{e:Qt} 
\end{equation}

Provided we can exchange the gradient and the expectation, following \eqref{e:nablaS_def},
we have,
\[ 
\frac{\partial}{\partial x_i} \Expect_x[f(X(t)) ]  
= \Expect_x \bigl[ [\tilgrad f(X(t))]_i  \bigr] ,
\quad 1\leq i\leq\ell,
\]
and consequently, the following elegant formula is obtained:
\begin{equation}
 \grad P^t f(x)  = \Expect_x [\tilgrad f(X(t))]  
= Q^t \grad f\,(x),\;\;\; x\in \Re^\ell.
\label{e:elegantQ}
\end{equation}
Justification requires minimal assumptions on the function $f$.
The proof of \Prop{t:elegantQ} is based on Lemmas~\ref{t:FellerGen}
and~\ref{t:Feller} contained in Appendix~\ref{s:proof_of_elegantQ}.

\begin{proposition}
\label{t:elegantQ}
Suppose that Assumptions~A1 and~A2 hold, 
and that $f^2$ and $\| \nabla f \|_2^2$ both lie in $\Lvx$.    
Then \eqref{e:elegantQ} holds,  and $ \grad P^t f\,(x)$ is continuous
as a function of $x\in\Re^\ell$.
\end{proposition}

\begin{proof}
The proof uses  \Lemma{t:Feller} in the Appendix, and a variant of the  truncation argument 
of \cite{devkonmey17a}.  
Let $\{ \Chi_n :  n\ge 1\}$ be a sequence
of functions satisfying, for each~$n$:
\begin{romannum}
\item $ \Chi_n $ is a continuous approximation to the indicator 
function on the set $R_n$, where
\[
R_n \eqdef \{ x \in\Re^\ell :  |x_i|\le n,\ \ 1\le i\le d \}\,,
\]
in the sense that  $0\le \Chi_n(x) \le 1$ for all $x$,
$ \Chi_n (x) = 1$ when $x \in R_n$,  and  $ \Chi_n (x) = 0$ when $x \in R_{n+1}^c$.

\item    $\nabla \Chi_n$ is continuous and uniformly bounded:   $\sup_{n,x} \|\nabla\Chi_n(x)\| <\infty$.   
\end{romannum}
On denoting $f_n = \Chi_n f$,  we have,
\[
\nabla f_n =  \Chi_n \nabla f  +  f \nabla \Chi_n,
\]
which is bounded and continuous under the assumptions of the proposition.
An application of the mean value theorem combined with 
dominated convergence allows us to exchange  
differentiation and expectation:
\[
\frac{\partial}{\partial x_i} \Expect_x[  f_n(X(t))]   =  
\Expect_x\Bigl[ \frac{\partial}{\partial x_i}   f_n(X(t)) \Bigr]\,,
\quad 1\leq i\leq\ell.
\]
This identity is equivalent to  
  \eqref{e:elegantQ}  for   $f_n$.   

Under the assumptions of the proposition there is a constant $b$ such 
that $\|\nabla f_n (x) \|^2\le b v (x)$ for each $n$ and $x \in \Re^\ell$.   Applying the 
dominated convergence theorem once more gives,
\[
 Q^t \grad f\,(x)
 = 
 \lim_{n\to\infty}
 Q^t \grad f_n\,(x)\,,\qquad x\in \Re^d.
\]
The limit is continuous by  \Lemma{t:Feller},
and an application of 
Lemma 3.6 of  \cite{devkonmey17a} completes the proof.
\end{proof}

\Prop{t:elegantQ}~$(a)$ strongly suggests the representation  
$\grad h_\beta = \Omega_\beta \grad c$ in \eqref{e:OmDream} holds,  with
\begin{equation} 
\Omega_\beta   \eqdef \sum_{t=0}^{\infty} \beta^t Q^t.
\label{e:Omegarep_disc}   
\end{equation}
This is indeed justified (under additional assumptions) 
in \cite[Theorem~2.4]{devkonmey17a},
and it forms the basis of the
$\nabla$-LSTD-learning algorithms developed in this paper.

Similarly, the representation 
$\grad h = \Omega \grad c$ with $\Omega=\Omega_1$
for the gradient of the relative value function $h$
is derived, under appropriate conditions, in 
\cite[Theorem~2.3]{devkonmey17a}.


\section{Differential LSTD-Learning}
\label{s:TD}

In this section we develop the new
{\em differential LSTD} (or $\grad$-LSTD, or `grad-LSTD') learning algorithms
for approximating
the value functions $h_\beta$ and $h$,
cf.~(\ref{e:DCOE_disc}),~(\ref{e:fish-sum}).
The algorithms are presented first, with supporting theory 
in \Section{s:derLSTD}. 
We concentrate mainly on the family of discounted-cost value 
functions $h_\beta$, $0<\beta<1$. The extension to the case of 
the relative value function $h$ is briefly 
discussed in Section~\ref{s:extend}.

\subsection{Differential LSTD algorithms}
 
We begin with a review of the standard 
Least Squares TD-learning
(LSTD) algorithm, cf.~\cite{bertsi96a,CTCN}.
We assume that the following are given:
A target number of iterations $T$
together with $T$ samples from the process 
$\bfmX$, the discount factor $\beta$,
the functions $\psi$, and a gain sequence $\{\alpha_t\}$.
Throughout the paper the gain sequence $\{\alpha_t\}$ 
is taken to be $\alpha_t = 1/t$, $t\geq 1$.

\begin{algorithm}[H]
\caption{\em Standard LSTD algorithm}
\label{LSTD}
\begin{algorithmic}[1]
\INPUT Initial $b(0),\varphi(0) \in \Re^d$, $M(0)$ $d\times d$
positive definite,
and $t=1$
\Repeat
\State $\varphi(t)  = \beta \varphi(t-1) +  \psi(X(t))$;
\State $b(t)  =   b(t-1) + \alpha_t \big( \varphi(t) c(X(t)) - b(t-1) \big)$;
\State $M(t) = M(t-1) + \alpha_t \big( \psi(X(t)) 
	\psi^\transpose(X(t)) - M(t-1) \big)$;
\State $t = t+1$
\Until{$t \geq T$}
\OUTPUT $\theta=M^{-1}(T)b(T)$
\end{algorithmic}
\end{algorithm} 

\noindent
Algorithm~\ref{LSTD} is equivalent to the LSTD($1$) algorithm of 
\cite{boy02}; see Section~\ref{s:dtdlambda} and \cite{devmey17a,devmey17b} 
for more details.

To simplify discussion we restrict to a stationary setting for the 
convergence results in this paper.

\begin{proposition}
\label{t:LSTDconverges}
Suppose that assumption~A1 holds, and that the functions 
$c^2$ and $\|\psi\|_2^2$ are in $\Lvx$.   Suppose moreover that the 
matrix $M=\Expect [\psi(X) \psi^\transpose(X)] \,, X \sim \pi$, is of full rank.

Then, there exists a version of the pair process 
$(\bfmX,\bfvarphi)=\{(X(t),\varphi(t))\}$ that is stationary 
on the two-sided time axis, 
and for any initial choice of  $b(0),\varphi(0)  \in\Re^d$  
and $M(0)$ positive definite,  Algorithm~\ref{LSTD} is consistent:
\[
\theta^* = \lim\limits_{t\to\infty} M^{-1}(t) b(t) \quad \text{a.s.},
\]
where $\theta^*$ is the least squares minimizer in \eqref{e:TDgoal}.
\end{proposition}

\begin{proof}
The existence of a stationary solution $\bfmX$ on the two-sided time interval follows directly from $v$-uniform ergodicity, and we then  define, for each $t \geq 0$,
\[
\varphi(t) =  \sum_{i=0}^\infty \beta^i   \psi(X(t-i)).
\]
The optimal parameter can be expressed 
$\theta^* =  M^{-1} b$ in which   
$b =\Expect[\varphi(t) c(X(t))]$, where the expectation is in steady state, so the result 
follows from the law of large numbers for this stationary ergodic
process.
\end{proof}

In the construction of the LSTD algorithm, the optimization 
problem \eqref{e:TDgoal} is cast as a minimum-norm
problem in the Hilbert space,
\[
L_2^\pi = \bigl\{ g 
\colon\Re^\ell\to\Re \ : \ \| g \|_\pi^2 
= \langle g,g \rangle_\pi <\infty\bigr\},
\]
with inner-product,
$\langle f,g \rangle_\pi \eqdef  \int  f(x)g(x) \pi(dx)$.  

The $\nabla$-LSTD algorithm presented next
is based on a minimum-norm problem in a 
different Hilbert space.
For $C^1$ functions $f$, $g$, with each $[\grad f]_{i}, \,\, [\grad g]_{i} \in L_2^\pi$, $1 \leq i \leq \ell$,
define the inner product, 
\[
\langle f,g \rangle_{\pi,1} = \int  \grad f (x)^\transpose \grad g (x) \pi(dx), 
\]
with the associated norm $\| f\|_{\pi,1} := 
\sqrt{\langle f,f \rangle_{\pi,1} }$. We let $L_2^{\pi,1}$ 
denote the set of functions with finite norm:
\begin{equation}
L_2^{\pi,1} = \bigl\{ h\colon\Re^\ell\to\Re 
\ : \ \|h\|_{\pi,1}^2  < \infty\bigr\}.
\label{e:L2pi1}
\end{equation}
Two functions $f,g\in L_2^{\pi,1} $ are considered identical if $\|f-g\|_{\pi,1} =0$.  In particular, this is true if the difference $f-g$ is a constant independent of $x$.

The `differential' version of the least-squares problem in~\eqref{e:TDgoal},
given as the nonlinear program~\eqref{e:gradTD}, can now be recast as,
\begin{equation}
\theta^* = \argmin_{\theta} \|h_\beta^\theta - h_\beta \|_{\pi,1}.
\label{e:OptThetaGhalphaDT}
\end{equation}
Given a target number of iterations $T$
together with $T$ samples from the process 
$\bfmX$, the discount factor $\beta$,
the functions $\psi$, and a gain sequence $\{\alpha_t\}$, the $\nabla$-LSTD algorithm, defined in Algorithm~\ref{dLSTD}, solves~\eqref{e:OptThetaGhalphaDT},  with
\begin{equation}
[ \grad \psi \, (x)]_{i,j} \eqdef  \frac{\partial }{\partial x_i} \psi_j (x),
\quad
x\in\Re^\ell.
\label{e:gradpsi}
\end{equation}


\begin{algorithm}[H]
\caption{\em $\grad$-LSTD algorithm}
\label{dLSTD}
\begin{algorithmic}[1]
\INPUT Initial $b(0)\in\Re^d$,
$\varphi(0) \in \Re^{\ell\times d}$, $M(0)$ $d\times d$
positive definite,
and $t=1$
\Repeat
\State $\varphi(t) 
	= \beta \dffA(t) \varphi(t-1) +  \grad \psi(X(t))$;
\State $b(t)= b(t-1) + \alpha_t \big( \varphi^\transpose(t) 
	\grad c(X(t)) - b(t-1) \big)$;
\State $M(t) = M(t-1)$\\
	\hfill
	$ + \alpha_t \big( ( \grad \psi(X(t)))^\transpose  
	\grad \psi(X(t))  -  M(t-1) \big)$;
\State $t = t+1$
\Until{$t \geq T$}
\OUTPUT $\theta=M^{-1}(T)b(T)$
\end{algorithmic}
\end{algorithm}  

\noindent
Once the estimate of $\theta^*$ is obtained
from Algorithm~\ref{dLSTD},
the required estimate of $h_\beta$ is obtained as
$h_\beta^{\theta} = \theta^\transpose \psi + \kappa (\theta)$,
where
\begin{equation}
\kappa (\theta) = -\pi(h_\beta^{\theta}) + \eta / (1-\beta)\,,
\label{e:kappatheta}
\end{equation}
with $\eta = \pi(c)$ as in  \eqref{e:barc},
and with the two means $\eta$ and $\pi(h_\beta^\theta)$
given by the results of the following recursive estimates: 
\begin{eqnarray} 
\barh_\beta(t) & = & \barh_\beta(t-1) + \alpha_t \big( h_\beta^{\theta(t)} - \barh_{\beta}(t-1) \big),
\label{e:dTD4}
\\	
\displaystyle
\eta(t) & = & \eta(t-1) + \alpha_t \big( c(X(t)) - \eta(t-1) \big).
\label{e:dTD5}
\end{eqnarray} 
It is immediate that $\eta(t) \to \eta$, a.s., as $t\to\infty$,
by the law of large numbers for $v$-uniformly ergodic Markov chains \cite{MT}.

 Steps \eqref{e:dTD4} and \eqref{e:dTD5} are not required if the approximation $h_\beta^{\theta(T)}$ is to be used just for obtaining the policy (or policy-improvement), which is the case in most control applications.

Convergence of the parameter estimates  is
established next.

\subsection{Derivation and analysis}  
\label{s:derLSTD}

In the notation of the previous section, 
and recalling 
the definition of $\grad \psi$ in \eqref{e:gradpsi},
we write:
\begin{eqnarray}
M &=& \Expect [(\grad \psi(X))^\transpose \grad \psi(X)] \,,
	\label{e:MDef}\\
b &=& \Expect \bigl[ ( \grad \psi(X) )^\transpose \grad h_\beta(X) \bigr]\,, \quad X \sim \pi.
\label{e:bDef}
\end{eqnarray}
\Prop{t:LSTDquad} follows immediately from these 
representations, and 
the definition of the norm $\|\cdot\|_{\pi,1}$.

\begin{proposition}
\label{t:LSTDquad}
The norm in \eqref{e:OptThetaGhalphaDT} is quadratic in $\theta$:    
\begin{equation}
\|h_\beta^\theta - h_\beta\|^2_{\pi,1} = \theta^\transpose M \theta - 2b^\transpose\theta + k
\label{e:QuadtraticRep}
\end{equation}
in which for each $1\le i, j\le d$,
\begin{equation}
\hspace{-.12cm}
M_{i,j} = \langle \psi_i, \psi_j \rangle_{\pi,1} \, ,
\hspace{.135cm}
 b_i = \langle \psi_i,  h_\beta \rangle_{\pi,1}   \,,
 \hspace{.135cm} 
 k = \langle h_\beta,h_\beta \rangle_{\pi,1}
\label{e:bMInitDef}
\end{equation} 
Consequently, the optimizer \eqref{e:OptThetaGhalphaDT}
is any solution to:
\begin{equation}
M \theta^* = b.
\label{e:OptimalTheta}
\end{equation}
\end{proposition}

As in the standard LSTD-learning algorithm, the representation 
for the vector $b$ in~(\ref{e:bDef})
involves the function $h_\beta$, which is unknown.  
An alternative representation will be obtained,
which is amenable 
to recursive approximation, and this
will form the basis of the $\nabla$-LSTD algorithm.  

Assumption~A3 is used to justify this representation:

\begin{Anum}{A3}
\item
\noindent
For   any $C^1$ functions $f,g$ 
satisfying $f^2, g^2\in\Lvx$ and 
$\|\grad f\|_2^2,\|\grad g\|_2^2\in\Lvx$, 
the following holds for the stationary version of the chain $\bfmX$:
\begin{equation}
\sum_{t=0}^{\infty}   \Expect  \Bigl[  \bigl| \grad f(X(t))^\transpose    \Sens(t)   \grad g(X(0)) \bigr| \Bigr]  < \infty.
\label{e:A3_dream}
\end{equation}

\item 
\noindent
The function $c$ is continuously differentiable, 
$c^2$ and $\|\nabla c\|^2_2  \in\Lvx$,  and for some $b_1 < \infty$,
$0<\rho_1 < 1$,
\[
\| Q^t \nabla c \, (x)  \| ^2  \le  b_1  \rho_1^t  v(x)  ,\quad x\in\Re^\ell,\,t\ge 0.
\]
\label{e:A3_nobigdeal}
\end{Anum}

\noindent
Theorem~2.1 of  \cite{devkonmey17a} establishes A3.\ref{e:A3_nobigdeal} under 
additional conditions on the model.  The bound \eqref{e:A3_dream} is related 
to the existence of a negative Lyapunov exponent for the Markov chain $\bfmX$ \cite{arnwih86}.

Under~A3 we can obtain a justification for the representation for the gradient of the value function:

\begin{lemma}
\label{t:OmExist}
Suppose that assumptions A1--A3 hold,  and that  
$c^2,\,\|\grad c\|_2^2 \in \Lvx$. Then the two representations 
in  \eqref{e:OmDream} hold $\pi$-a.s.:
\[
\grad h_\beta = \Omega_\beta \grad c
\quad \textit{and} \quad
\grad h = \Omega \grad c.
\]
\end{lemma}

\begin{proof}
\Prop{t:elegantQ} justifies the following calculation,
\[
\begin{aligned}
\nabla h_{\beta,n}(x) &\eqdef  \nabla \left(\sum_{t=0}^n 
\beta^t P^t c(x)\right)
=  \sum_{t=0}^n \beta^t  Q^t \nabla c(x),
\end{aligned}
\]
and also implies that this gradient is continuous as a function of $x$.
Assumption A3.\ref{e:A3_nobigdeal} implies that the right-hand side 
converges to $\Omega_\beta \nabla c\, (x)$ as $n\to\infty$.
The function $\Omega_\beta \nabla c$   is continuous in $x$,  
since  the limit is uniform on compact subsets of 
$\Re^\ell$ (recall that $v$ is continuous).   
Lemma~3.6 of  \cite{devkonmey17a} then completes the proof.   
\end{proof}

A stationary realization of the algorithm is established next. 
\Lemma{t:gradLSTDstat} follows immediately from the assumptions: 
The non-recursive expression for $\varphi(t)$ in \eqref{e:EliVecDT} 
is immediate from the recursions in~Algorithm~\ref{dLSTD}.

\begin{lemma}
\label{t:gradLSTDstat} 
Suppose that assumptions~A1--A3 hold, and that   
$\|\psi\|_2^2$ and  $\|\grad \psi\|_2^2$ are in $\Lvx$.    
Then there is a version of the pair process $(\bfmX,\bfvarphi)$ 
that is stationary on the two-sided time line, 
and for each $t\in\ZZ$,
\begin{equation}
\varphi(t) = \sum_{k=0}^{\infty} \beta^k \big[ 
\Theta^{t-k} \Sens(k) \big] \gradpsi(X(t-k)), 
\label{e:EliVecDT}
\end{equation} 
where 
$\Theta^{t-k} \Sens(k)=
\dffA(t)\dffA(t-1)\cdots\dffA(t-k+1)$.
\end{lemma}

The remainder of this section consists of a proof of the following proposition which establishes the convergence of the $\nabla$-LSTD algorithm.

\begin{proposition}
\label{t:dLSTDconverges}
Suppose that assumptions~A1--A3 hold, and that 
$c^2$, $\| \grad c\|^2_2$,  $\|\psi\|_2^2$ and  $\|\grad \psi\|_2^2$ are 
in $\Lvx$. Suppose moreover that the matrix $M$ in \eqref{e:MDef} 
is of full rank. Then, for the stationary  process $(\bfmX,\bfvarphi)$, 
the $\nabla$-LSTD-learning algorithm is consistent: 
For any initial $b(0)  \in\Re^\ell$  and   $M(0) >0$,  
\[
\theta^* = \lim\limits_{t\to\infty} M^{-1}(t) b(t)  \quad \text{a.s.},
\]
where $\theta^*$ is the least squares minimizer in \eqref{e:gradTD}. Moreover, with probability one,
\[
\eta=
\lim_{t\to\infty}  \eta(t) ,\quad \pi(h_\beta^{\theta^*}) = 
\lim_{t\to\infty} \barh_\beta(t) ,
\]
and hence $\lim_{t\to\infty} \{
-\barh_\beta(t)  +  \eta(t) /(1-\beta)\} =\kappa(\theta^*)$.
\end{proposition}

We begin by obtaining alternative representations for $b$ defined in \eqref{e:bMInitDef}. The proof of the following \Lemma{t:bRep} is contained in Appendix~\ref{s:proof_of_bRep}.

\begin{lemma}
\label{t:bRep}
Under the assumptions of \Prop{t:dLSTDconverges},  
\begin{equation}
\begin{aligned}
b^\transpose &= \sum_{t=0}^{\infty} \beta^t    \Expect \bigl[ \bigl ( \Sens^\transpose(t) \grad c(X(t)) \bigr)^\transpose \gradpsi(X(0))   \bigr]    
\\
&=  \Expect \Bigl [ \big(\grad c(X(0))\big)^\transpose \varphi(0) \Bigr ]\, .
\end{aligned}
\label{e:bpsi_disc_init}
\end{equation}
\end{lemma}


\paragraph*{Proof of \Prop{t:dLSTDconverges}}   \Lemma{t:bRep} combined with the stationarity assumption implies that,
\[
\begin{aligned}
\lim_{T\to\infty } \frac{1}{T} b(t) &=
\lim_{T\to\infty } \frac{1}{T} \sum_{t=1}^T 
\varphi^\transpose(t) \grad c(X(t))  
\\
& = \Expect[\varphi^\transpose(0) \grad c(X(0)) ] =b.
\end{aligned}
\]
Similarly,  for each $T\ge 1$ we have,
\[
M(T) = M(0) + \sum_{t=1}^{T} (\gradpsi(X(t)))^\transpose \gradpsi(X(t)),
\]
and by the law of large numbers we once again obtain:
\[
\lim_{T\to\infty } \frac{1}{T}  M(T) = M.
\]
Combining these results establishes  $\theta^* = \lim\limits_{t\to\infty} M^{-1}(t) b(t)$.

Convergence of $ \{\eta(t)\}$ in \eqref{e:dTD5} is identical,  and convergence of $\{\barh_\beta(t)\}$ in \eqref{e:dTD4} also follows from the law of large numbers since we have convergence of $\theta(t)$.  
\qed

\subsection{Extension to average cost}
\label{s:extend}

The $\nabla$-LSTD recursion of Algorithm~\ref{dLSTD} is also consistent
in the case $\beta=1$, which corresponds to the relative value function $h$ 
in place of the discounted-cost value function $h_\beta$.
Although we do not repeat the details of the analysis here, we observe that 
nowhere in the proof of \Prop{t:dLSTDconverges} do we use the assumption 
that $\beta<1$. Indeed, it is not difficult to establish that,
under the conditions of the proposition, the $\nabla$-LSTD-learning 
algorithm is also convergent with $\beta=1$, and
that the limit solves the quadratic program:
\[
\theta^* = \argmin_{\theta} \|h^\theta - h \|_{\pi,1}.
\]

 
\section{Differential LSTD($\lambda$)-Learning}
\label{s:dtdlambda}

In this section we introduce a \emph{Galerkin approach} 
for the construction of new {\em differential LSTD$(\lambda)$}
(or $\grad$-LSTD($\lambda$), or `grad'-LSTD($\lambda$)) 
algorithms. 
The relationship between TD-learning algorithms and the Galerkin relaxation 
has a long history; see \cite{urlinhan94,jaajorsin94,pin97} 
and \cite{tsiroy97a}, and also \cite{yuber10,ber11,sze11} for more
recent discussions.

The algorithms developed here offer approximations
for the value functions $h_\beta$ and $h$ associated
with a cost function $c$ and a Markov chain~$\bfmX$,
under the same conditions as in Section~\ref{s:TD}.
Again, we concentrate on the discounted-cost value 
functions $h_\beta$, $0<\beta<1$. The extension
to the relative value function $h$ is straightforward,
following along the same lines as
in Section~\ref{s:extend}, and thus omitted.

The starting point of the development of the Galerkin approach 
in this context is the Bellman equation~(\ref{e:DPEquDisc}).
Since we want to approximate the gradient
of the discounted-cost value function
$h_\beta$, it is natural to begin with
the `differential' version
of~(\ref{e:DPEquDisc}), i.e., taking gradients on both sides of \eqref{e:DPEquDisc},
\begin{equation}
\grad c+\beta Q \grad h_\beta-\grad h_\beta=0,
\label{e:GradDPEquDisc}
\end{equation}
where we used the identity `$\grad P=Q\grad$' from
\Prop{t:elegantQ}~$(a)$. Equivalently, 
using the definitions of $Q$ and $\dffA$ in
terms of the sensitivity process,
this can be
restated as the requirement that the steady state expectations,
$$
\Expect \Big[\!
Z(t)\!
\Big(
\grad c(X(t))
+ \beta \dffA^\transpose(t+1) \grad  h_\beta (X(t+1)) 
-\!\grad h_\beta (X(t)) 
\!\Big)\! \Big]
$$
are identically equal to zero, for a `large enough' 
class of random matrices $Z(t)$.

The Galerkin approach is simply a relaxation of this requirement:
A specific $(\ell \times d)$-dimensional, stationary 
process $\bfelig=\{\elig(t):t\geq 0\}$ will be constructed,
and the parameter $\theta^*\in\Re^d$ 
which achieves,

\medskip

\noindent
\begin{equation}
\displaystyle{\Expect\Big[ \elig^{\transpose}(t)   
\Big( \grad c(X(t)) + \beta \dffA^\transpose(t+1)
\grad  h_\beta^{\theta^*} (X(t+1))}
-\grad h_\beta^{\theta^*} (X(t)) 
\Big)   \Big] = 0,\!\!
\label{e:MetricGalGradTD} 
\end{equation}  
will be estimated, where the above expectation 
is again in steady state.
By its construction, $\bfelig$ will be adapted to~$\bfmX$.
We call $\bfelig$ the sequence of \textit{eligibility matrices},  
borrowing language from the standard LSTD($\lambda$)-learning 
literature \cite{sutbar98,bertsi96a,sze10}.

Motivation for the minimum-norm criterion \eqref{e:OptThetaGhalphaDT} is 
clear, but algorithms that solve this problem often suffer from high 
variance. The Galerkin approach is used because it is simple, generally 
applicable, and it is observed that the variance of the 
algorithm is often significantly reduced with $\lambda<1$.

It is important to note, as we also discuss below, 
that the process $\bfelig$ will depend on the value of
$\lambda$, so the LSTD($\lambda$) (respectively,
$\grad$-LSTD($\lambda$)) algorithms with different $\lambda$
will converge to different parameter values
$\theta^*=\theta^*(\lambda)$, satisfying the corresponding
versions of~(\ref{e:MetricGalGradTD}).



\subsection{Differential LSTD($\lambda$) algorithms}

Recall the standard algorithm introduced in \cite{boy02};
see also \cite{CTCN,bertsi96a}. Given
a target number of iterations $T$
together with $T$ samples from the process 
$\bfmX$, the discount factor $\beta$,
the functions~$\psi$, a gain sequence $\{\alpha_t\}$,
and $\lambda\in[0,1]$:

\begin{algorithm}[H]
\caption{\em Standard LSTD($\lambda$) algorithm}
\label{LSTDla}
\begin{algorithmic}[1]
\INPUT Initial $b(0),\zeta(0) \in \Re^d$, $M(0)$ $d\times d$
positive definite,
and $t=1$
\Repeat
\State 
	$\elig(t) = \beta \lambda \elig(t-1) +  \psi(X(t))$;
\State 
	$b(t) = (1-\alpha_t)  b(t-1) + \alpha_t \elig(t) c(X(t))$;
\State 
	$M(t) = (1-\alpha_t)  M(t-1)$\\
	\hfill
	$ + \alpha_t \elig(t) 
	\bigl[\psi(X(t)) - \beta \psi(X(t+1)) \bigr]^\transpose$;
\State $t = t+1$
\Until{$t \geq T$}
\OUTPUT $\theta=M^{-1}(T)b(T)$
\end{algorithmic}
\end{algorithm}  

The asymptotic consistency of Algorithm~\ref{LSTDla}
is established, e.g., in \cite{boy99,boy02}.
Note that, unlike in Algorithms~\ref{LSTD} and~\ref{dLSTD},
here there is no guarantee that $M(t)$ is positive definite
for all $t$, so by the output value
of $\theta=M^{-1}(T)b(T)$ we mean that obtained 
by using the pseudo-inverse of~$M(T)$; and
similarly for Algorithm~\ref{gLSTDla} presented next.

The differential analog of Algorithm~\ref{LSTDla} is very similar;
recall the definition of $\grad \psi\, (x)$ in~\eqref{e:gradpsi}.


\begin{algorithm}[H]
\caption{\em $\grad$-LSTD($\lambda$) algorithm}
\label{gLSTDla}
\begin{algorithmic}[1]
\INPUT Initial $b(0)\in \Re^d$, $\zeta(0) \in \Re^{\ell\times d}$, 
$M(0)$ $d\times d$
positive definite,
and $t=1$
\Repeat
\State 
	$\elig(t) = \beta \lambda \dffA(t) \elig(t-1) +  \grad \psi(X(t))$;
\State 
	$b(t) = (1-\alpha_t)b(t-1) + \alpha_t\elig^\transpose(t)\grad c(X(t))$;
\State 
	$M(t)  = (1-\alpha_t) M(t-1)$\\
	\hfill
	$ + \alpha_t \bigl[\grad \psi(X(t)) 
	- \beta \dffA^\transpose(t+1) 
	\grad \psi(X(t+1)) \bigr]^\transpose  \elig(t)$;
\State $t = t+1$
\Until{$t \geq T$}
\OUTPUT $\theta=M^{-1}(T)b(T)$
\end{algorithmic}
\end{algorithm}  

As with Algorithm~\ref{dLSTD}, 
after obtaining the estimate of $\theta$ 
from Algorithm~\ref{gLSTDla},
the required estimate of $h_\beta$ is formed 
based on the recursions in 
equations~(\ref{e:kappatheta}),~(\ref{e:dTD4})
and~(\ref{e:dTD5}).

\subsection{Derivation and analysis}
For any $\lambda\in[0,1]$,
the parameter vector $\theta^*=\theta^*(\lambda)$ that 
solves~\eqref{e:MetricGalGradTD} 
is a Galerkin approximation to the  exact solution which 
solves the fixed point equation~(\ref{e:GradDPEquDisc}).

The proof of the first part of 
\Prop{t:GradLSTDGquad} below
follows from the assumptions. In particular,
the non-recursive expression for $\zeta(t)$
is a consequence of the recursions in
Algorithm~\ref{gLSTDla}.
The proof of the second part of 
the proposition 
follows 
from~\eqref{e:MetricGalGradTD}.

\begin{proposition}
\label{t:GradLSTDGquad}
Suppose that assumptions A1--A3 hold, and that
$\|\psi\|^2$ and $\|\grad\psi\|_2^2$ are in~$\Lvx$.
Then:

$(i)$
There is a stationary version of the pair process
$(\bfmX,\bfelig)$ 
on the two-sided time axis, and for 
each $t\in\ZZ$ we have,
$$\zeta(t)=\sum_{k=0}^\infty(\beta\lambda)^k\big[\Theta^{t-k}
\Sens(k)\big]\nabla\psi(X(t-k)),$$
where $\Theta^{t-k} \Sens(k)=
\dffA(t)\dffA(t-1)\cdots\dffA(t-k+1)$.

$(ii)$
The optimal parameter vector $\theta^*$ that satisfies \eqref{e:MetricGalGradTD} is any solution to $M \theta^* = b$,
in which,
\begin{equation}
M  = \Expect [  \bigl(\grad \psi(X(t)) - \beta \dffA^\transpose(t+1) 
\grad \psi(X(t+1)) \bigr)^\transpose  \elig(t)   ],
\label{e:MDef_Lambda}
\end{equation}
and,
\begin{equation}
b =  \Expect \bigl[ (\elig(t))^\transpose  \grad c(X(t)) ],
\label{e:bDef_Lambda}
\end{equation}
where the expectations are under stationarity.
\end{proposition}

The following then follows from the law of large numbers:
\begin{proposition}
\label{t:dLSTDlambdaconverges}
Suppose that the assumptions of \Prop{t:dLSTDconverges} hold. Suppose moreover that the matrix $M$ appearing in \eqref{e:MDef_Lambda} 
is of full rank. Then,  for each initial conditions $b(0)  \in\Re^d$ and $M(0) \in \Re^{d \times d}$,  
the {\em $\nabla$-LSTD($\lambda$)} Algorithm~\ref{gLSTDla}
is consistent:
\[
\lim\limits_{t\to\infty} M^{-1}(t) b(t) = \theta^*  \quad \text{a.s.},
\]
where $\theta^*=\theta^*(\lambda)$ solves \eqref{e:MetricGalGradTD}.  
   
This limit holds both for the stationary version  $(\bfmX,\bfvarphi)$ 
defined in     \Prop{t:GradLSTDGquad},   and
also for $\varpi$-almost all initial 
$(X(0),\elig(0) )$, where $\varpi$ denotes the marginal for the 
stationary version  $(\bfmX,\bfvarphi)$.   
\end{proposition}

\subsection{Optimality of $\nabla$-LSTD($1$)}
\label{s:gradoptimal}

Although different values of $\lambda$ in LSTD($\lambda$) lead
to different parameter estimates $\theta^*=\theta^*(\lambda)$,
it is known that in the case $\lambda=1$ the parameter estimates 
obtained using the standard LSTD($\lambda$) algorithm converge 
to the solution to the minimum-norm problem \eqref{e:TDgoal},
cf.~\cite{tsiroy97a,devmey17a}. Similarly, it is shown here that 
the parameter estimates obtained using the 
$\nabla$-LSTD($1$) algorithm converge to the solution of 
the minimum-norm problem \eqref{e:OptThetaGhalphaDT}.

\begin{proposition}
\label{t:dLSTDlambdaEqdLSTD}
Suppose that the assumptions of \Prop{t:dLSTDconverges} hold. Then, 
the sequence of parameters $ \bftheta=\{\theta(t)\}$ obtained using the 
{\em $\nabla$-LSTD($1$)} Algorithm~\ref{gLSTDla},
converges 
to the solution of the minimum-norm problem \eqref{e:OptThetaGhalphaDT}.
\end{proposition}

\begin{proof}
From \Prop{t:dLSTDlambdaconverges}, the estimates $\bftheta$ obtained using 
the $\nabla$-LSTD($\lambda$) algorithm converge to  $ \theta^* = M^{-1}b$, 
where $M$ and $b$ are defined in \eqref{e:MDef_Lambda} and 
\eqref{e:bDef_Lambda}, and $\elig(t)$ defined by the 
recursion in Algorithm~\ref{gLSTDla}. 
It remains to be shown that this coincides with the parameter  that solves \eqref{e:OptThetaGhalphaDT} in the case $\lambda = 1$.

Substituting
the identity,
\begin{equation}
\elig(t+1)   = \beta \dffA(t+1) \elig(t) + \grad \psi(X(t+1)),
\label{e:eliglambda1}
\end{equation}
in \eqref{e:MDef_Lambda},
gives the following representation,
\[
\begin{aligned}
M 
& =  - \beta \Expect[ \bigl (\dffA^\transpose(t+1) 
	\grad \psi(X(t+1)) \bigr)^\transpose  \elig(t)  ] 
	+ \Expect [  \bigl (\grad \psi(X(t)) \bigr )^\transpose \elig(t) ] \\ 
& =  - \beta \Expect[ \bigl (\dffA^\transpose(t+1) 
	\grad \psi(X(t+1)) \bigr)^\transpose  \elig(t)  ]
	+ \beta \Expect [  \bigl (\dffA^\transpose(t) 
	\grad \psi(X(t)) \bigr )^\transpose \elig(t-1) ] \\ 
& \qquad\qquad 
	+ \Expect [  \bigl (\grad \psi(X(t)) \bigr )^\transpose 
	\grad \psi(X(t))  ] \\ 
& =  
	\Expect[\bigl (\grad \psi(X(t)) \bigr )^\transpose \grad \psi(X(t))],
\end{aligned}
\]
where the last equality is obtained using time stationarity   of $\bfmX$.
Therefore, the matrix $M$ obtained using the $\nabla$-LSTD($1$) 
algorithm coincides with the matrix $M$ of 
$\nabla$-LSTD. 

\begin{figure*}
\Ebox{1}{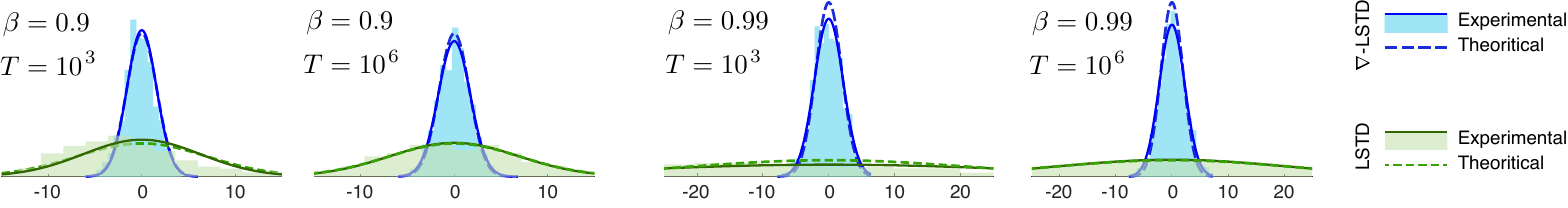}
\caption{Histograms of $\sqrt{T}\big(\theta_2(T) - \theta^*_2 \big)$ produced
by the LSTD-learning and the $\nabla$-LSTD-learning algorithms, 
after $T=10^3$ and   $T=10^6$ iterations.
The \textit{experimental} bell-curves are those obtained using {\tt histfit} in Matlab using parameter estimates from multiple runs of each algorithm.
Computation of the theoretical variance is obtained in Appendix~\ref{s:var_ana_LSTD}. -- see \eqref{e:nabla_TD_avar}, and the discussion that follows.  
} 
\label{f:HistLinear}	
\end{figure*}

To obtain the required representation for $b$,  recall 
that  $\elig(t) \equiv \varphi(t)$,  where the former 
is defined in \eqref{e:eliglambda1} and the latter in the
recursion of Algorithm~\ref{dLSTD}.  
Applying \Lemma{t:bRep}, 
it follows that
the vector $b$ of the $\nabla$-LSTD(1) algorithm \eqref{e:bDef_Lambda} 
coincides with the vector $b$ of $\nabla$-LSTD 
algorithm \eqref{e:bpsi_disc_init}. 
\end{proof}

\section{Numerical Results}
\label{s:sim}

Collected together here are results from several
numerical experiments, which illustrate the general
theory of the previous sections and also suggest possible
extensions.

Since, under general conditions, all estimates considered 
obey a central limit theorem
\cite{kusyin97}, we use the \textit{asymptotic variance} to be the 
primary figure of merit in evaluating performance.
The relevant variances are estimated by collecting data from many independent  runs of each   algorithm.

Specifically, we show comparisons between the performance 
achieved by LSTD, $\nabla$-LSTD, LSTD($\lambda$) and 
the $\nabla$-LSTD($\lambda$) algorithms. In examples where there 
is regeneration (the 
Markov chain $\bfmX$ visits some state infinitely often),
we replace the LSTD algorithm with the regenerative 
LSTD algorithm of \cite{CTCN,huachemehmeysur11}. The regenerative algorithm is found to have reduced variance in these experiments.   

The standard TD($\lambda$) algorithm was also considered, but in 
all examples its variance was found to be  several orders of 
magnitude greater than alternatives.   The matrix-gain variant with minimal  
asymptotic variance 
  is precisely LSTD($\lambda$) \cite{boy02,devmey17a}.
This  was found to have better performance,  and was therefore used for comparisons;
the reader is referred to Section~2.4 of \cite{devmey17a} for details on the 
relationship between TD($\lambda$) and LSTD($\lambda$) algorithms, 
and their asymptotic variances.
 
We also consider two extensions of $\nabla$-LSTD for a specific example: 
The approximation of the relative value function for the speed-scaling model of \cite{chehuakulunnzhumehmeywie09}.  First, for this reflected process evolving on $\Re_+$, it is shown that the sensitivity process $\bfSens$ can be defined, subject to conditions on the dynamics near the boundary.   Second,  the algorithm is tested in a discrete state space setting.  There is no apparent justification for this approach, but it   performs remarkably
well in simulations.

\subsection{Linear model}
\label{s:lin_mod}

A scalar linear model offers perhaps the clearest illustration 
of the performance of the $\nabla$-LSTD-learning algorithm, 
demonstrating its
superior convergence rate compared to the standard
LSTD algorithm.

Consider the scalar linear process  
\[
X(t+1) = a X(t) + N(t+1),\quad t\geq 1,
\]
where $a\in(0,1)$ is a constant 
and $\bfmN$ is  i.i.d.\ $\clN(0,1)$.  The cost function is taken to be quadratic, $c(x)=x^2$,
and for the basis of the approximating function class is chosen as $\psi(x) = (1,x^2)^\transpose$.
The true value function $h_\beta$ turns out to also be 
quadratic and symmetric, which means that it 
can be expressed exactly 
in terms of $\psi$,
as $h_\beta=h_\beta^{\theta^*}$,
with,
$$
h_\beta^{\theta^*}(x)=\sum_j\theta_j\psi_j(x)=\theta_1^*+\theta_2^* x^2,
$$
for appropriate $\theta^*\in\Re^2$; cf.~(\ref{e:hLinearPar}).
The
constant term, $\theta_1^*$, 
can be estimated 
as $\kappa(\theta)$ using~\eqref{e:kappatheta} 
in the $\nabla$-LSTD algorithm.
Therefore, the interesting part of the problem 
is to estimate the optimal value of the second parameter,
$\theta_2^*$.

For this linear model, the first recursion for the 
$\nabla$-LSTD Algorithm~\ref{dLSTD} becomes,
\begin{equation}
\varphi(t) = \beta a \varphi(t-1) +  \grad \psi(X(t)),
\label{e:dLSTDLinear}
\end{equation}  
while in the LSTD Algorithm~\ref{LSTD},
\begin{equation}
\varphi(t) =  \beta \varphi(t-1) +    \psi(X(t)).
\label{e:LSTDLinear}
\end{equation}
Although both of these algorithms are consistent, there
are two differences which immediately suggest that 
the asymptotic variance 
of $\nabla$-LSTD should be much smaller
than that of LSTD.  First,
the additional discounting factor $a$ appearing in \eqref{e:dLSTDLinear}, 
but absent in \eqref{e:LSTDLinear}, is the reason 
why the asymptotic variance of the $\grad$-LSTD is bounded
over $0<\beta<1$,
whereas that of the standard LSTD 
grows without bound as $\beta \to 1$. 
Second,
the gradient reduces the growth rate of each function of $x$;
in this case, reducing the quadratic growth of $c$ and $\psi$ 
to the linear growth of their derivatives.

In the numerical experiments surveyed here we use $a=0.7$,
and   two different discounting factors: 
$\beta=0.9$, and $\beta=0.99$. 
The optimal parameters can be computed explicitly,
giving
$\theta^*=(16.1, 1.79)^\transpose$ when $\beta=0.9$,
and $\theta^*=(192.27, 1.9421)^\transpose$ when $\beta=0.99$.
The histogram of the estimated value of $\theta_2$ was
computed based on 1000 repetitions of the
same experiment, where the output of each algorithm 
was recorded after $T=10^3$ and after $T=10^6$ iterations.
The results are shown in Figure~\ref{f:HistLinear}.


With $\beta=0.9$ it was found that the variance of $\theta_2$  using the standard LSTD algorithm is about ten times the variance using the  $\nabla$-LSTD algorithm.  Consequently,     $\nabla$-LSTD-learning 
is about $10$~times faster than LSTD in this example.
This difference in performance grows with
larger~$\beta$, as observed on the two histograms on the right hand side of   Figure~\ref{f:HistLinear}.

In conclusion, in contrast to the standard LSTD algorithm,
the asymptotic variance of 
$\nabla$-LSTD in this example is bounded uniformly 
over $0<\beta<1$, and the algorithm can also be used 
to estimate the relative value function \eqref{e:ACOE_disc2}.

 We next consider an example with   non-linear dynamics. 

\subsection{Dynamic speed scaling}
\label{s:dyn_ss}

Dynamic speed scaling refers to control techniques for power management 
in computer systems. The goal   is to control the processing speed so as 
to optimally balance energy and delay costs; 
this can be done by reducing (or increasing) the processor speed at times 
when the workload is small (respectively, large). For our present purposes,
speed scaling is a simple stochastic control problem, namely, 
a single-server queue with controllable service rate.

This example was considered in \cite{chehuakulunnzhumehmeywie09} with the goal of minimizing the average cost \eqref{e:barc}. Approximate policy iteration was used to obtain the optimal control policy, and a regenerative form of the LSTD-learning was used to provide an approximate relative value function $h$ 
at each iteration.

The underlying discrete-time Markov decision process model is as follows:   
At each time $t$,  the state $X(t)$ is the (not necessarily
integer valued) queue length,
which can also be interpreted more generally as the size of the 
workload in the system;  $N(t) \geq 0$ is number of job arrivals;
and $U(t)$ is the service completion at time $t$, which is
subject to the constraint $0 \leq U(t) \leq X(t)$.  The evolution equation is the controlled random walk: 
\begin{equation}
X(t+1) = X(t) - U(t) + N(t+1)\, , \quad t \geq 0.
\label{e:CRW}
\end{equation}  
Under the assumption that $\bfmN$ is i.i.d.\  and that
$\bfmU=\{U(t)\}$ is 
obtained using a state feedback policy, $U(t) = \fee(X(t))$,  
the controlled model  is a Markov chain of the form \eqref{e:SP_disc}.


In the experiments that follow in Sections~\ref{s:exp}
and~\ref{s:geo} we consider the problem of approximating
the relative value function $h$, for a fixed state feedback policy $\fee$,
so $\beta=1$ throughout.
We consider the cost function $c(x,u) = x+u^2/2$, and feedback law
$\fee$ given by,
\begin{equation}
\fee(x) = \min\{x,1+\epsy \sqrt{x}\},\quad x\in\Re,
\label{e:EpsilonPolicy}
\end{equation}
with $ \epsy > 0$.  This is similar in form to the optimal average-cost 
policy computed in \cite{chehuakulunnzhumehmeywie09},
where it was shown 
that the value function is well-approximated by 
the function $h^\theta(x) = \theta^\transpose \psi(x)$ 
for some $\theta\in\Re^2_+$, and $\psi(x) = (x^{3/2} ,x)^\transpose$.  
As in the linear example, the gradient  $\grad \psi(x) = (\frac{3}{2} x^{1/2} , 1)^\transpose$ has slower growth as a function of $x$.

On a more technical note, 
we observe that
implementation of the $\nabla$-LSTD algorithms requires attention 
to the boundary of the state space:
The sensitivity process $\bfSens$ defined in~\eqref{e:SensDef} requires 
that the state space be open, and that the dynamics are smooth.  
Both of these assumptions are violated in this example.  However, 
with $X(0)=x$, 
we do have a representation for the right derivative,
$\Sens(t) \eqdef \partial^+ X(t)/\partial x,$
which evolves according to the recursive equation,
\begin{equation}
\Sens(t+1)  = \dffA(t+1) \Sens(t)  = 
\big[1- \ddxp \fee\, (X(t))\big]\Sens(t),
\label{e:SensSS}
\end{equation}
where the `$+$' again denotes right derivative. 
Therefore, we adopt the convention,
\begin{equation}
\dffA(t+1)  = 1 - \ddxp \fee\, (X(t))
\label{e:diffAss}
\end{equation} 

We begin with the case in which the marginal of 
$\bfmN$ is exponential.  In this case the right 
derivatives and ordinary derivatives coincide a.e.
The regenerative LSTD algorithm used in 
\cite{chehuakulunnzhumehmeywie09} 
is not applicable in this case because there is no state 
that is visited infinitely often with probability one. We therefore restrict our comparisons to the LSTD($\lambda$) algorithms.

\subsubsection{Exponential arrivals}
\label{s:exp}

Suppose the $N(t)$ are i.i.d.\ Exponential$(1)$ 
random variables, and that
$\bfmX$ evolves on $\Re_+$ 
according to (\ref{e:CRW}) and~\eqref{e:EpsilonPolicy}.
The derivatives $\dffA(t)$
in~\eqref{e:diffAss} become, 
\begin{equation}
\dffA(t+1) 
=  \one\{X(t) > \bar{\epsy}\} 
\bigl[ 1 - \half \epsy X(t)^{-1/2} \bigr] \,, 
\label{e:diffASS}
\end{equation}
where  
$\bar{\epsy} = \half ( \epsy + \sqrt{\epsy^2 + 4})$ and
$\one $ is the indicator function.

For the implementation
of the $\nabla$-LSTD Algorithm~\ref{dLSTD},
we note that 
the recursion for $\bfvarphi$,
\begin{equation} 
\varphi(t+1) = \dffA(t+1) \varphi(t) +  \grad \psi(X(t+1)),
\label{e:RdTD3}
\end{equation}
regenerates: Based on \eqref{e:diffASS}, 
$\varphi(t+1)  =  \grad \psi(X(t+1))$  when $X(t)\le \bar{\epsy}$.
The second recursion in Algorithm~\ref{dLSTD} becomes,
\[
b(t+1) = b(t) + \alpha_{t+1} \bigl( -b(t) 
+ \grad c(X(t+1)) \varphi(t+1) \bigr),
\]
in which,   
\begin{equation}
\begin{aligned}
\label{e:gradcSS}
\grad c(X(t)) & = 1 + \fee(X(t)) \grad \fee(X(t)),
\\
\textstyle
\grad \fee(X(t))  & = \one\{X(t) \leq \fee(X(t))\} +
\frac{\epsy}{2} \frac{1}{\sqrt{X(t)}}  \one\{X(t) > \bar{\epsy}\} .
\end{aligned}
\end{equation}

Implementation of the $\nabla$-LSTD($\lambda$) Algorithm~\ref{gLSTDla}
uses similar modifications, with $\{\dffA(t)\}$ and 
$\{\grad c (X(t))\}$ obtained 
using \eqref{e:diffASS} and \eqref{e:gradcSS}.

Various forms 
of the TD($\lambda$) algorithms with $\lambda \in [0, 1)$ were implemented for comparison, but as reasoned in \Section{s:Intro}, all of them appeared to have infinite asymptotic variance.
Implementation of the LSTD($\lambda$) algorithm 
resulted in improved performance. Since this is an average-cost problem, Algorithm~\ref{LSTDla} must be  modified slightly \cite{tsiroy99a,boy02, CTCN}:

\begin{algorithm}[H]
\caption{\em LSTD($\lambda$) algorithm for average cost}
\label{LSTDlaavg}
\begin{algorithmic}[1]
\INPUT Initial $\eta(0) \in \Re^+ $, $b(0),\varphi(0), \eta_\psi(0) \in \Re^d$, $M(0)$ $d\times d$
positive definite,
and $t=1$
\Repeat
\State $\eta(t) = (1- \alpha_t) \eta(t - 1) +  \alpha_{t} c(X(t))$
\State $ \eta_\psi(t) = (1-\alpha_t) \eta_\psi(t-1) + \alpha_t \psi(X(t)) $
\State $\tilpsi(t) \eqdef \psi(X(t)) - \eta_\psi(t)$
\State $\elig(t)  = \lambda \elig(t-1) +  \tilpsi(X(t))$;
\State $b(t)  =   (1- \alpha_t) b(t-1) + \alpha_t \elig(t) \big( c(X(t))-\eta(t) \big)$;
\State $M(t) = (1-\alpha_t)  M(t-1)$\\
	\hfill
	$ + \alpha_t \elig(t) \bigl[\tilpsi(X(t)) - \tilpsi(X(t+1)) \bigr]^\transpose$;
\State $t = t+1$
\Until{$t \geq T$}
\OUTPUT $\theta=M^{-1}(T)b(T)$
\end{algorithmic}
\end{algorithm}

Other than taking $\beta=1$, the
main difference between Algorithms~\ref{LSTDla} and~\ref{LSTDlaavg} is that 
we have replaced the cost function $c(X(t))$ with its centered version,
$c(X(t)) - \eta(t)$, where $\eta(t)$ is the estimate of the average cost 
after $t$ iterations. While this is standard for average cost problems, 
we have similarly replaced the basis function $\psi$ with $\tilpsi$ 
to restrict the growth rate of the eligibility vector $\elig(t)$, 
which in turn reduces the variance of the estimates 
$\bftheta=\{\theta(t)\}$. This is justified because the approximate 
value functions $h^\theta_a = \theta^\transpose \tilpsi$ differs 
from $h^\theta_b = \theta^\transpose \psi$ only by a constant term, 
and the relative value function is unique only up to additive constants. 
Experiments where $\psi$ was used instead of $\tilpsi$ resulted 
in worse performance.

\Figure{fig:TDvsGTD_Expo1_EpsyPoint5} shows the histogram of the estimates
for $\theta_1$ and $\theta_2$ 
obtained using $\nabla$-LSTD-learning, LSTD($0$)-learning, and 
$\nabla$-LSTD($\lambda$)-learning, $\lambda=0$ and $0.5$, 
after $T=10^5$ time steps.

\begin{figure}[h]  
\includegraphics[width=1\hsize]{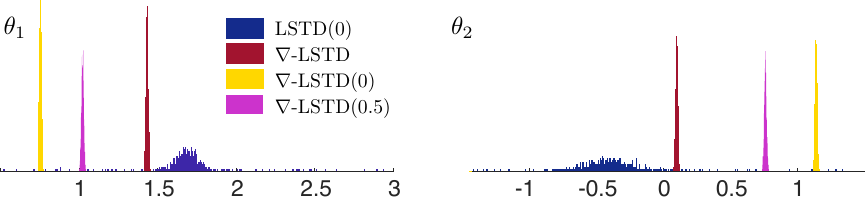}
\caption{Histograms of the parameter estimates obtained using 
the LSTD and $\nabla$-LSTD algorithms after $T = 10^5$ iterations, 
under the stationary policy \eqref{e:EpsilonPolicy} with $\epsy = 0.5$; 
$\bfmN$ is i.i.d. exponential.}
\label{fig:TDvsGTD_Expo1_EpsyPoint5}
\end{figure}

As noted earlier in Section~\ref{s:gradoptimal},
we observe that, as expected, 
different values of $\lambda$ lead
to different parameter estimates $\theta^*(\lambda)$,
for both the LSTD($\lambda$) and the $\grad$-LSTD($\lambda$)
classes of algorithms.


\subsubsection{Geometric arrivals}
\label{s:geo}

In \cite{chehuakulunnzhumehmeywie09}, the authors consider a discrete state 
space, with  $N(t)$ geometrically distributed on an integer lattice $\{0, \DeltaA, 2\DeltaA,\dots\}$, $\Delta > 0$.
In this case, the  theory developed for the $\nabla$-LSTD algorithm does not fit the model since we have no convenient representation of a sensitivity process. 
Nevertheless, the algorithm can be run by replacing gradients with ratios of  differences.  In particular, in implementing the algorithm we 
substitute the definition \eqref{e:diffASS} with,
\[
\dffA(t) = 1 -   [ \fee(X(t) + \DeltaA) - \fee(X(t))  ]/ \DeltaA,
\] 
and $\grad c$ is approximated similarly.  
For the distribution of $N(t)$ we take, 
$\Prob (N(t)=n\DeltaA) = (1-p_A)^{n} p_A$;
the values $p_A = 0.04$ and $\DeltaA = 1/24$ were chosen, so that $\Expect[N(t)] = 1$.


The sequence of steps followed in the regenerative LSTD-learning algorithm 
are similar to Algorithm~\ref{LSTD} \cite{chehuakulunnzhumehmeywie09, CTCN}:
\begin{algorithm}[H]
\caption{\em Regenerative LSTD algorithm for average cost}
\label{ReLSTD}
\begin{algorithmic}[1]
\INPUT Initial $\eta(0) \in \Re^+ $, $b(0),\varphi(0), \eta_\psi(0) \in \Re^d$, $M(0)$ $d\times d$
positive definite,
and $t=1$
\Repeat
\State $\eta(t) = (1- \alpha_t) \eta(t - 1) +  \alpha_{t} c(X(t))$
\State $ \eta_\psi(t) = (1-\alpha_t) \eta_\psi(t-1) + \alpha_t \psi(X(t)) $
\State $\tilpsi(t) \eqdef \psi(X(t)) - \eta_\psi(t)$
\State $\varphi(t)  = \one\{X(t-1) \neq 0\}  \varphi(t-1) +  \tilpsi(X(t))$;
\State $b(t)  =   (1- \alpha_t) b(t-1) + \alpha_t \varphi(t) \big( c(X(t))-\eta(t) \big)$;
\State $M(t) = (1- \alpha_t) M(t-1) + \alpha_t \big( \tilpsi(X(t)) 
\tilpsi^\transpose(X(t)) \big)$;
\State $t = t+1$
\Until{$t \geq T$}
\OUTPUT $\theta=M^{-1}(T)b(T)$
\end{algorithmic}
\end{algorithm}

\noindent
The eligibility vector $\varphi(t)$ regenerates
(i.e., resets to $0$) every time the queue empties. 
The regenerative LSTD($\lambda$) algorithm is obtained by making 
similar modifications, namely,
replacing Line $5$ of Algorithm~\ref{LSTDlaavg} with: 
\[
\elig(t)  = \one\{X(t-1) \neq 0\} \lambda \elig(t-1) +  \tilpsi(X(t)).
\]

\Figure{fig:VarAnaEpsyPt5T10e5} shows the histogram of $\theta(T)$ obtained 
using the regenerative LSTD, LSTD($0$), $\nabla$-LSTD, $\nabla$-LSTD($0$), 
and $\nabla$-LSTD($0.5$) algorithms, after $T=10^5$ iterations.
Observe that, again, the variance of the parameters obtained using 
the $\nabla$-LSTD algorithms 
is extremely small compared to the LSTD algorithms. 

\begin{figure}[h]
\includegraphics[width=1\hsize]{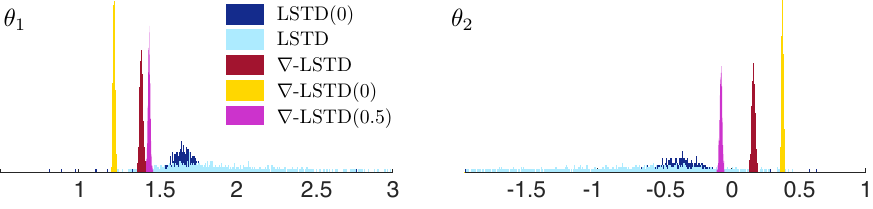}
\caption{Histograms of the parameter estimates obtained using the LSTD and 
$\nabla$-LSTD algorithms after $T=10^5$ iterations, under the stationary policy \eqref{e:EpsilonPolicy} with $\epsy = 0.5$; $\bfmN$ is i.i.d.\ geometric.}
\label{fig:VarAnaEpsyPt5T10e5}
\end{figure}

It is once again noticeable in 
\Figure{fig:VarAnaEpsyPt5T10e5} that,
as before in the results shown in 
\Figure{fig:TDvsGTD_Expo1_EpsyPoint5},
different values for $\lambda$ 
lead to different parameter estimates.
To compare  performance, 
the  relative \emph{Bellman error} was 
computed:
\[
\clE_B(x) = [P - I]h(x) + c(x) - \eta(T),
\]
where $P$ of course depends on the policy $\fee$,   
$h = {\bar{\theta}}^\transpose \psi$, where ${\bar{\theta}}$ is the mean 
of the $10^3$ parameter estimates obtained for each of the different algorithms, and $\eta(T)$ denotes the estimate of the average cost $\eta$ using $T = 10^5$ samples.
\Figure{TilPsiAndTilC_EpsyPoint5} shows plots of $\clE_B(x)$
for each of the five algorithms, 
for typical values of $\theta(T)$, with 
$T=10^3$, $10^4$ and $10^5$.
Once again, the feedback policy \eqref{e:EpsilonPolicy}  
was used, with $\epsy = 0.5$. 

\begin{figure*}
\centering
\Ebox{1}{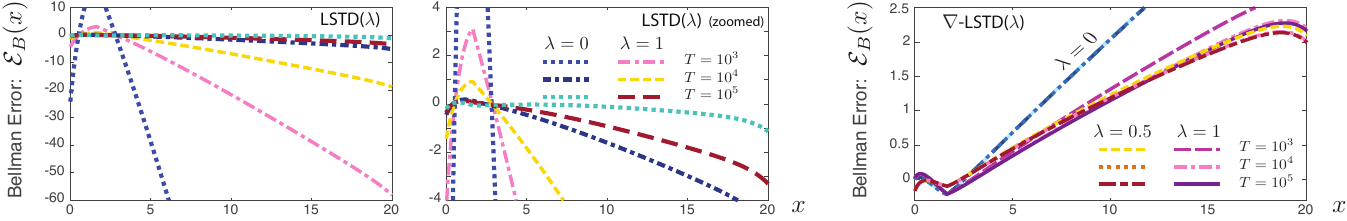}
\caption{Bellman error corresponding to the estimates of $h$ 
using LSTD algorithms (left) and $\nabla$-LSTD algorithms (right).}
\label{TilPsiAndTilC_EpsyPoint5}
\end{figure*}

The Bellman error of
the $\nabla$-LSTD algorithms   appears to have
converged after $T = 10^3$ iterations,  
and the limit is nearly zero for the range of $x$
where the stationary distribution has non-negligible
mass.
Achieving similar performance using the LSTD 
algorithms requires more than $T = 10^5$ iterations.


\section{Conclusions}
\label{s:conclusions}

The new gradient-based temporal difference learning algorithms introduced 
in this work prove to be excellent alternatives to the more classical 
TD- and LSTD-learning methods for value function approximation. 
In the examples considered, the algorithms show remarkable capability 
to reduce the variance. There are two known explanations for this: 
\begin{romannum}
\item The magnitude of the functions that are used as inputs to the $\nabla$-LSTD algorithms are smaller compared to those in the case of LSTD algorithms; for example, if the basis functions for LSTD are polynomials of order $n$, then the basis functions for $\nabla$-LSTD will be of the order $n-1$. 
\item  There is an additional ``discounting'' factor that is inherent in the $\nabla$-LSTD algorithms, due to the derivative sequence $\{\dffA(t)\}$. For example, in the simple linear model experiment (cf. \Section{s:lin_mod}), we had $\dffA(t) \equiv a$, for some $a<0$, and when this term multiplies the original discount factor $\beta$, it can cause a significant reduction in the growth rate of the eligibility trace.
\end{romannum}

The introduction of $\lambda$ is to further reduce variance.  However, the optimal parameter vector obtained using the $\grad$-LSTD($\lambda$) algorithm will not in general solve the minimum-norm problem \eqref{e:OptThetaGhalphaDT} if $\lambda<1$.
As in the case of classical LSTD($\lambda$) algorithms, one can expect a bias versus variance trade-off in choosing $\lambda$ for the $\grad$-LSTD($\lambda$) algorithms. It is conjectured that as $\lambda \to 0$, the bias becomes larger, and perhaps the variance reduces. A bound similar to \eqref{e:tsiroy97a_Theorem1} is a topic of future research.

Though we only consider problems that involve ultimately estimating the value function for a fixed policy, estimating the gradient of the value function has its own applications: 

{\em State estimation. } In \cite{yanmehmey11}, the authors are interested in estimating the gradient of the relative value function, which is useful in obtaining the innovation gain for a non-linear filter.  

{\em Control. } When one is interested in optimizing the policy using policy iteration in a continuous state space setting, the gradient of the value function could be more useful than the value function, in the policy update step. 

{\em Mean-field games. }  As was recently emphasised in \cite{bensou18},
``...{\em it is not the Bellman equation 
that is important, but the gradient of its solution}.''  That is, it is the 
gradient of the value functions that is the critical quantity of interest 
in computation of solutions to mean field games.
This appears to indicate that the techniques in this paper might offer 
computational tools 
for approximating solutions in this class of optimal control problems, 
and in particular in applications to power systems.

Perhaps the biggest drawback of the $\nabla$-LSTD algorithms is the requirement of the knowledge of $\clA(t)$ defined in \eqref{e:dffA}.  In certain problems (such as in the queuing example discussed in \Section{s:dyn_ss}) this information is directly available.  Other applications 
may require a combination of system identification with the $\nabla$-LSTD algorithm. 



There are many other directions in which this work can be extended. Perhaps the most interesting open question is why the algorithm is so effective even in a discrete state space setting in which there is no theory to justify its application. It will also be worth exploring algorithms analogous to $\nabla$-LSTD, which use finite-differences instead of gradients in a discrete state space setting.

We are currently considering the extension of the $\nabla$-LSTD algorithms to a continuous time setting. Though the algorithms are straightforward to obtain, the convergence theory will require extensions of the theory of \cite{devkonmey17a}.  

Finally, it will  be interesting to see how the techniques developed here 
could be used to estimate the gradient of the state-action 
value function (either the Q-function of Watkins \cite{wat89} or SARSA \cite{rumnir94}). 
This will greatly simplify application to  control.

\section*{Acknowledgments}


A.M.D. and S.P.M. were supported from ARO grant W911NF1810334.  Additional financial support from the National Science Foundation,  EPCN 1935389 \&\ CPS~1646229,   and from a graduate fellowship from the University of Florida Informatics Institute is gratefully acknowledged. 

I.K. was supported by the Hellenic Foundation for Research and Innovation (H.F.R.I.) under the ``First Call for H.F.R.I. Research Projects to support Faculty members and Researchers and the procurement of high-cost research equipment grant,'' project number 1034.


\appendix  

\begin{center}
\Large{\bf Appendices}
\end{center}


\section{Proof of~\Prop{t:elegantQ}}
\label{s:proof_of_elegantQ}

Here we state and prove two simple technical lemmas
that are needed for the proof of~\Prop{t:elegantQ}.
Let $v$ denote the Lyapunov function in Assumption~A1.

\begin{lemma}
\label{t:FellerGen}
Let $R$ denote a transition kernel that has the Feller property
and satisfies,
for some $B_0<\infty$:
\[
Rv\, (x)  \eqdef \int R(x,dy) v(y)  \le B_0 v(x)  \,, \quad x\in\Re^\ell.
\]
And let $Z$ be a kernel that is absolutely continuous with respect to $R$,  
with density $\xi\colon \Re^\ell\times \Re^\ell \to \Re$ such that,
\[
Zg\, (x)   = \int R(x,dy) \xi(x,y) g(y)\,, \quad x\in \Re^\ell,  
\]
for any bounded measurable function $g \colon\Re^\ell \to\Re$.

If the density is continuous and for some $\delta\in (0,1)$,
\[
B_\xi \eqdef 
\sup_{x,y} \frac{|\xi(x,y)|}{v^\delta(y)}  <\infty,
\]
then $Z$ has the Feller property:  
$Zg$ is continuous whenever $g$ is bounded and continuous.  
\end{lemma}

\begin{proof}
The proof is based on a truncation argument:     
Consider the sequence of closed sets,
\[
S_n = \{ x \in\Re^\ell :  v(x) \le n \},
\quad\,n\geq 1.
\]
Take any  sequence of
continuous
 functions  $\{ \Chi_n :  n\ge 1\}$ satisfying  $0\le \Chi_n(x) \le 1$ for all $x$,
$ \Chi_n (x) = 1$ on $S_n$,  and  $ \Chi_n (x) = 0$ on $S_{n+1}^c$.   Hence $\Chi_n$ is   a continuous approximation to the indicator on $S_n$.   

Denote $g_n = g\Chi_n$ for a given bounded and continuous function  $g$.   
The function $Z g_n$ is continuous because  $\xi (x,y) g_n(y)$ is bounded and continuous.   It remains to show that $Zg = \lim_{n\to\infty}  Z g_n$,  and that the convergence is uniform on compact sets.  

Under the assumptions of the lemma,   for each $x$,
\[
\begin{aligned} 
| Zg\, (x) -  Zg_n\, (x) |  
	&  \le \| g\|_\infty  \int R(x,dy)  [1-\Chi_n(y)] |\xi(x,y)|
\\
  	& \le B_\xi  \| g\|_\infty  \int_{S_n^c}  R(x,dy)   v^\delta(y) .
\end{aligned}
\] 
Since $v(y)>n$ on $S_n^c$,
 this gives, for all $x$,
\[
| Zg\, (x) -  Zg_n\, (x) |    
\le \frac{1}{n^{1-\delta}}  B_0  B_\xi  \| g\|_\infty    v(x).
\]
It follows that  $ Z g_n\to Zg $ uniformly on compact sets,
since $v$ is assumed to be continuous.   
\end{proof}

\begin{lemma}
\label{t:Feller}
Subject to Assumptions A1 and A2,
\begin{romannum}
\item  $P^t f$ is continuous if   $|f|^2\in\Lvx$ and $f$ is continuous.
\item  The vector-valued function  $Q^t \nabla f $ is continuous,    provided  
$\nabla f$ is continuous, 
$|f|^2\in\Lvx$,  and
$\| \nabla f \|_2^2\in\Lvx$. 
\end{romannum}
\end{lemma}

\begin{proof}
Both parts follow from \Lemma{t:FellerGen},  with  $R = P^t$.   
The bound $P^t v\le B_0 v$ holds under~A1,  and in fact the 
constant $B_0$ can be chosen independent of $t$.

For part~(i), choose $\xi(x,y) = f(y)$.  The Feller property for the 
kernel $Z$ defined in \Lemma{t:FellerGen}  implies in particular that $Zg$ is continuous when $g\equiv 1$.  $Zg= P^t f$ in    
 this special case. 
For part~(ii), observe that each $Q^t_{i,j}$ 
$1\leq i,j\leq\ell$, admits a continuous 
and bounded density by its definition, 
cf.~(\ref{e:SP_sens_disc}),~(\ref{e:Qt}):
\[
Q^t_{i,j}(x,dy)   =P^t(x,dy)  q^t_{i,j}(x,y).
\]
So, 
we have for each $i$ and $x$,
\[
[Q^t \nabla f\, (x) ]_i  = \sum_j   \int  P^t(x,dy)  q^t_{i,j}(x,y)   [\nabla f(y)]_j  \,.
\]
Fix $i,j$ and let $\xi(x,y) :=  q^t_{i,j}(x,y)   [\nabla f(y)]_j$.    
Then Lemma~\ref{t:FellerGen}
implies that the $(i,j)$-term in the last sum,
$ \int  P^t(x,dy)  q^t_{i,j}(x,y)   [\nabla f(y)]_j$,
is continuous in $x$. 
\end{proof}

\section{Proof of \Lemma{t:bRep}}
\label{s:proof_of_bRep}

The following shift-operator on sample space is defined 
for a stationary version of $\bfmX$: For a random variable 
of the form,
\[
Z = F(X(r),N(r), \dots, X(s), N(s))\,,\quad \text{with $r\le s$,}
\]
we denote, for any integer $k$,
\[
\Theta^k Z = F(X(r+k),N(r+k), \dots, X(s+k), N(s+k)).
\]
Consequently, viewing $\Sens(t)$ as a function
of $\dffA(1),\ldots,\dffA(t)$ as in
the evolution equation \eqref{e:SP_sens_disc},
we have:
\begin{equation}
\Theta^k \Sens(t) =  \dffA(t+k)\cdots\dffA(2+k)\dffA(1+k) .
\label{e:shift}
\end{equation}

The representation \eqref{e:OmDream} for $\grad h_\beta$ is valid under 
Assumption~A3, by Lemma~\ref{t:OmExist}.
Using this and \eqref{e:SP_sens_disc} gives the first representation in \eqref{e:bpsi_disc_init}:
\begin{equation}
\begin{aligned}
b^\transpose  &= \int \Expect_x \big [(\Omega_\beta \grad c(x))^\transpose \gradpsi(x) \big ]\pi(dx) 
\\
&= \sum_{t=0}^{\infty} \beta^t  \int \Expect_x \big [( \Sens^\transpose(t)  \grad c(x))^\transpose \gradpsi(x) \big ]\pi(dx)
\\
&= \sum_{t=0}^{\infty} \beta^t    \Expect \bigl[ \bigl ( \Sens^\transpose(t) \grad c(X(t)) \bigr)^\transpose \gradpsi(X(0))   \bigr] \, .
\label{e:bpsi_disc_init2}
\end{aligned}
\end{equation}  
Stationarity implies that for any $t,k\in\ZZ$,
\begin{equation*}
\begin{aligned}
\Expect \Bigl[ \Bigl (&\Sens^\transpose(t)  \grad c(X(t)) \Bigr)^\transpose \gradpsi(X(0)) \Bigr] 
\\
&=  \Expect \Bigl[ \Bigl ( [\Theta^k\Sens^\transpose(t) ]\grad c(X(t+k)) \Bigr)^\transpose \gradpsi(X(k)) \Bigr] \, .
\end{aligned}
\end{equation*} 
Setting $k=-t$, the first representation in \eqref{e:bpsi_disc_init} becomes:
\[
\begin{aligned}
b^\transpose  &= \sum_{t=0}^{\infty} \beta^t \Expect \Bigl[ \big(\grad c(X(0))\big)^\transpose \big( \Theta^{-t} \Sens(t) \big)  \gradpsi(X(-t)) \Bigr]
\\
&= \Expect \Bigl[\big(\grad c(X(0))\big)^\transpose \Bigl(  \sum_{t=0}^{\infty} \beta^t  \big( \Theta^{-t} \Sens(t) \big) \gradpsi(X(-t)) \Bigr) \Bigr]\,,
\end{aligned}
\]
where last equality   is obtained under Assumption~A3 by applying Fubini's theorem. 
This combined with \eqref{e:EliVecDT}
completes the proof.
\qed


%
%
%
%
%
%
%
%
%
%
%

\section{Variance Analysis of $\grad$-LSTD($\lambda$) Algorithms}
\label{s:var_ana_LSTD}


As in the case of the classical LSTD($\lambda$) algorithms, the $\grad$-LSTD($\lambda$) algorithms also belong to a more general class of root-finding algorithms known as stochastic approximation (SA). Following the theory for variance analysis of \emph{linear} SA recursions, under slightly stronger conditions than \Prop{t:dLSTDconverges} and \Prop{t:dLSTDlambdaEqdLSTD}, the asymptotic variance of the  $\grad$-LSTD($\lambda$) algorithm is given by the following expression \cite{kusyin97,bor08a,devmey17a}:
\begin{equation}
\begin{aligned}
\Sigma_\theta 
& \eqdef \lim_{T \to \infty} T \Expect \big[ \big (\theta(T) - \theta^* \big ) \big (\theta(T) - \theta^* \big )^\transpose \big] 
\\
& = M^{-1} \Sigma_\Delta \big( M^{-1} \big)^\transpose
\label{e:nabla_TD_avar}
\end{aligned}
\end{equation}
where, the matrix $M$ is defined in \eqref{e:MDef_Lambda}, and   the ``noise covariance matrix"  is defined as follows: define 
\begin{equation}
\Delta(t) \eqdef \tilM(t) \theta^* + \tilb(t)  
\label{e:delta_t}
\end{equation}
where (with the quantities defined in Algorithm~\ref{dLSTD}),
\begin{equation}
\begin{aligned}
\tilM(t) & \eqdef M - \bigl[\grad \psi(X(t)) - \beta \dffA^\transpose(t+1)  \grad \psi(X(t+1)) \bigr]^\transpose  \elig(t)
\\
\tilb(t) & \eqdef b - \elig^\transpose(t)\grad c(X(t)) 
\end{aligned}
\end{equation}
with $b$ is defined in \eqref{e:bDef_Lambda},   $\theta^* = M^{-1} b$,  and all stochastic processes are assumed stationary. 
The noise covariance matrix   is then expressed by the two equivalent formula:
\begin{equation}
\Sigma_\Delta = \lim_{T\to \infty} \frac{1}{T} \Expect\big[ S(T) S(T)^\transpose \big] = 	\sum_{t = -\infty}^{\infty} R(t)
\label{e:sigma_delta}
\end{equation}
where 
\begin{equation}
R(t) = \Expect \big[ \big(\tilM(t) \theta^* - \tilb(t) \big) \big(\tilM(0) \theta^* - \tilb(0) \big)^\transpose \big], \,\quad t\ge 0\,,
\label{e:R_t}
\end{equation} 
with $R(-t)=R(t)^\transpose$,   and $S(T) = \sum_{t = 0}^{T - 1} \Delta(t)$.

An asymptotic variance formula for LSTD($\lambda$) algorithms can be obtained in a straightforward manner, using the definitions of $M$, $M(t)$, $b$, and $b(t)$ as in Algorithm~\ref{LSTD} and \Prop{t:LSTDconverges}.

In the following, we show how these expressions can be used to calculate the asymptotic variance of the two  algorithms when applied to the simple linear model described in Section~\ref{s:lin_mod}, and how the algorithms compare with each other with respect to this quantity.

\section{Asymptotic Variance for the Linear Model}

Consider the application of $\nabla$-LSTD and LSTD algorithms to estimate the value function for the linear model that is analyzed in Section~\ref{s:lin_mod}:
\begin{equation}
X(t+1) = a X(t) + N(t+1),\quad t\geq 1,
\label{e:linmodel}
\end{equation}
with the cost function defined to be a quadratic: $c(x) = x^2$.

The one dimensional basis function for $\nabla$-LSTD is given by $\gradpsi(x) = 2 x$, so that the estimate of the value function is $\nabla h^\theta_\beta(x) = 2 \theta x$\footnote{The constants that appear when taking derivatives are retained so that the parameter estimates obtained using the $\nabla$-LSTD algorithms can be compared to the ones obtained using the LSTD algorithm}. In this case, $M$ and $b$ are scalar quantities, and  $\bfvarphi \equiv \bfelig$ is a scalar sequence: 
\[
M = 4 \Expect [X^2 (t)] \qquad b = 2 \Expect[ X(t) \varphi(t) ]
\]
with $\varphi(t)$   defined in \eqref{e:dLSTDLinear}, and expectations in steady state.

Using the fact that $M \theta^* = b$, the auto-correlation function $R(t)$ defined in \eqref{e:R_t} is given by:
\[
\begin{aligned}
R(t) & \!=\! \Expect \big [  \big (4 \theta^* X^2 (t)  - 2 X(t) \varphi(t) \big ) (4 \theta^* X^2 (0)  - 2 X(0) \varphi(0) \big ) \big]
\end{aligned}
\]
In the case $\{N(t): t \geq 1\}$ is i.i.d Gaussian with mean $0$ and variance $\sigma_N^2$, using \eqref{e:linmodel}, it can be shown that the above expression simplifies to:
\begin{equation}
\begin{aligned}
R(t) = & \big( (\theta^*)^2 + 4 - 4 \theta^* \big) \Expect \big[ X^2(t) X^2(0) \big] 
\\
& - 2 \theta^* a \beta \sum_{s = 0}^{t-1}  (a\beta)^s a^{t-s} \Expect \big[ X^2(s) X^2(0) \big]
\end{aligned}
\label{e:R_t_nabla_LSTD_LQR}
\end{equation}
where, for each $t \geq 1$:
\begin{equation}
\Expect[X^2(t) X^2(0)] = a^{2t} \Expect [X^4(t)] + \big(\sum_{s = 0}^{t-1} a^{2s} \big) \sigma_N^2 \Expect [ X^2 (t) ]
\end{equation}
and the steady state expectations  are given by:
\begin{equation}
\Expect[X^2(t)] \eqdef \frac{\sigma_N^2}{1 - \alpha^2} \hspace{0.2in} \Expect[X^4(t)] \eqdef \frac{\sigma_N^2}{1-\alpha^4} \Big( \frac{6}{1 - \alpha^2} + 3 \Big)
\end{equation}
The noise covariance $\Sigma_{\Delta}$ is then obtained using the second definition in \eqref{e:sigma_delta}. The plots for theoretical values of the variance $\Sigma_\theta$ in Figure~\ref{f:HistLinear} is obtained using this expression.

Similarly, for the LSTD($\lambda$) algorithm, since $\psi(x) = (1,\, x^2)^\transpose$, we have
\[
M = \Expect \begin{bmatrix} 
1 & X^2(t) \\
X^2(t) & X^4(t)
\end{bmatrix} 
\qquad 
b = \Expect \big[X^2(t) \varphi(t) \big]
\]
with $\varphi(t)$ recursively defined in \eqref{e:LSTDLinear}. The auto-correlation can then be obtained using \eqref{e:R_t}, where:
\begin{equation}
\begin{aligned}
\tilM(t) & \eqdef M -\psi(X(t)) \psi^\transpose(X(t))
\\
\tilb(t) & \eqdef b - \varphi(t)c(X(t))
\end{aligned}
\end{equation}
and $\theta^* = M^{-1} b$.
However, the precise expression for $R(t)$ becomes much more complicated than the one obtained in \eqref{e:R_t_nabla_LSTD_LQR}, since it involves calculation of moments up to eighth order. 

An alternative way to obtain a theoretical formula for $\Sigma_\Delta$ is using the first definition in \eqref{e:sigma_delta}. For   fixed $T$ large enough, one can approximately obtain $\Sigma_{\Delta}$ by estimating the expectation in \eqref{e:sigma_delta} using Monte-Carlo:
\[
\Expect[S(T) S(T)^\transpose ] \approx \frac{1}{N} \sum_{i = 0}^{N} S^{(i)}(T) \big( S^{(i)}(T) \big)^\transpose\,,
\] 
The plots for theoretical values of the variance $\Sigma_\theta$ in Figure~\ref{f:HistLinear} were obtained using this approximation.

\newpage

\bibliographystyle{plain}

\end{document}